%% file: main.tex
\pdfoutput=1
\documentclass[twocolumn]{article} 
\usepackage{simpleConference}

\input{math_commands.tex}

\usepackage{url}
\usepackage{amsmath}
\usepackage{amssymb}
\usepackage{amsthm}
\usepackage{mathtools}
\usepackage{subfig}
\usepackage{color}
\usepackage{algorithm}
\usepackage{algorithmic}
\usepackage{times}
\usepackage{listings}
\usepackage{float}

\def\dD{\mathcal{D}}

\newtheorem{definition}{Definition}
\newtheorem{remark}{Remark}
\newtheorem{lemma}{Lemma}
\newtheorem*{lemma*}{Lemma}
\newtheorem*{claim*}{Claim}
\newtheorem{theorem}{Theorem}

\usepackage[round]{natbib}

\usepackage{hyperref}

\begin{document}
\date{\today}
\title{Optimal Defenses Against Gradient Reconstruction Attacks}

\author{\textbf{Yuxiao Chen} \\
Peking University\\
2100010776@stu.pku.edu.cn\\
\and
\textbf{Gamze Gürsoy} \\
Columbia University\\
gamze.gursoy@columbia.edu\\
\and
\textbf{Qi Lei} \\
New York University\\
ql518@nyu.edu\\
}

\maketitle

\begin{abstract}
Federated Learning (FL) is designed to prevent data leakage through collaborative model training without centralized data storage. However, it remains vulnerable to gradient reconstruction attacks that recover original training data from shared gradients. To optimize the trade-off between data leakage and utility loss, we first derive a theoretical lower bound of reconstruction error (among all attackers) for the two standard methods: adding noise, and gradient pruning. We then customize these two defenses to be parameter- and model-specific and achieve the optimal trade-off between our obtained reconstruction lower bound and model utility.
Experimental results validate that our methods outperform Gradient Noise and Gradient Pruning by protecting the training data better while also achieving better utility. The code for this project is available \href{https://github.com/cyx78/Optimal_Defenses_Against_Gradient_Reconstruction_Attacks}{here}.
\end{abstract}

\section{Introduction} Recent advancements in machine learning have led to remarkable achievements across multiple domains. These successes are driven largely by the ability to gather vast, diverse datasets to train these powerful models. However, this can be challenging to obtain in certain sectors such as healthcare and finance due to concerns about privacy and institutional restrictions. 

Federated or Collaborative Learning (FL)~\citep{mcmahan2023communicationefficientlearningdeepnetworks} has emerged as a solution to these concerns. Federated learning is a machine learning approach where multiple institutions or devices collaboratively train a model while keeping their data localized. Instead of sharing raw data, each participant shares model updates, aggregated centrally to create a global model that benefits from all participants' insights without compromising data privacy. The assumption is that summary-level information about the model (whether the model weights or the intermediate gradients) contains less information about the trained data.

However, FL is not immune to privacy risks, one type of attack that may harm privacy is the Gradient Reconstruction Attack (GRA), where adversaries attempt to reconstruct original training data from the shared gradients. Methods such as DLG \citep{zhu2019deepleakagegradients}, CAFE \citep{jin2022cafecatastrophicdataleakage}, and GradInversion \citep{yin2021gradientsimagebatchrecovery} have shown the feasibility of these attacks. To mitigate these risks, several defense mechanisms have been proposed. The most common approach is perturbing the gradients, such as DP-SGD \citep{abadi16deep} and Gradient Pruning \citep{zhu2019deepleakagegradients}.  Although these methods offer some level of data protection, they often encounter a trade-off between maintaining privacy and preserving model performance \citep{zhang2023tradingprivacyutilityefficiency}. 
\begin{figure}
    \centering
    \includegraphics[width=1\linewidth]{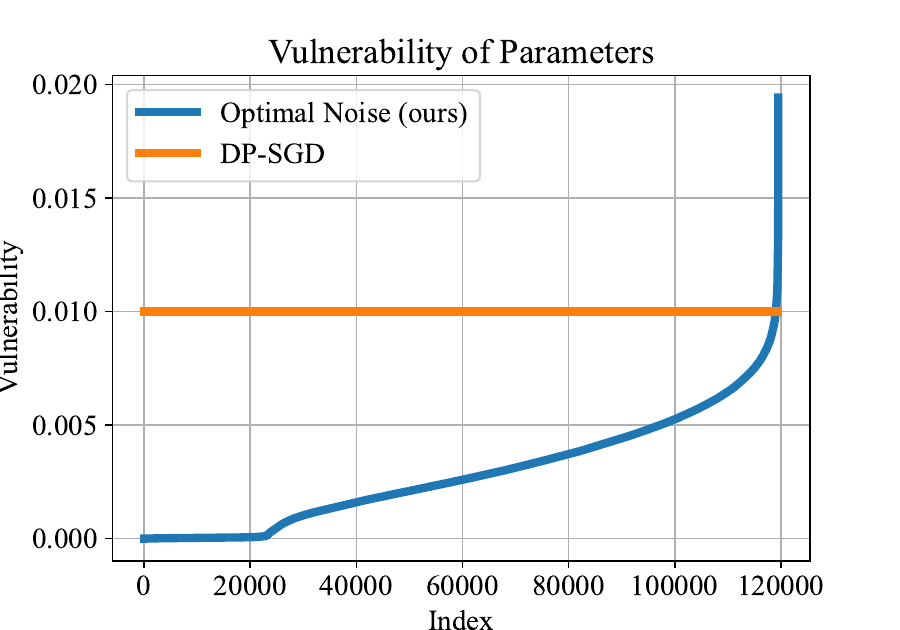}
    \vspace{-15pt}
    \caption{DP-SGD treats all parameters with the same vulnerability, while our method distinguishes the vulnerability of each parameter and designs a customized defense strategy.}
    \label{fig:introduction}
\end{figure}

In this work, we develop a new defense mechanism based on generalizing existing gradient perturbation methods to an optimal and parameter-specific defense. As also noted by \citet{shi-etal-2022-selective}, a universal defense strategy provides undifferentiated protection and is not optimal for utility-privacy trade-offs.
We customize defenses for each parameter and examine how these adjustments impact both privacy and utility, as illustrated in Figure \ref{fig:introduction}. Our objective is to optimize the balance between a model's resilience to gradient reconstruction attacks and its training effectiveness. Our primary contributions are: 
\begin{itemize} 
\item We establish a theoretical lower bound for the expected reconstruction error, which can be easily evaluated. 
\item We propose two defense mechanisms—Optimal Gradient Noise and Optimal Gradient Pruning—that maximize this bound for a given level of utility. 
\end{itemize}

In Section \ref{sec:background}, we provide a brief overview of federated learning and theoretical backgrounds for our method. In Section \ref{sec:theory}, we present the theoretical foundation for our lower bound and optimal defense methods, and present the implementations of our proposed algorithms. Section \ref{sec:experiments} evaluates their effectiveness against gradient reconstruction attacks in image classification tasks.

\subsection{Related Works} 
Federated learning (FL) was introduced by \citet{mcmahan2023communicationefficientlearningdeepnetworks} as a framework for collaborative model training without centralized data storage. Differential privacy (DP) \citep{10.1007/11681878_14,10.1007/11761679_29,10.1561/0400000042} has been used to define privacy of the algorithm. \citep{abadi16deep} introduced a differential private SGD algorithm to provide DP guarantees to the trained model. \citet{stock2022defendingreconstructionattacksrenyi} used R\'enyi differential privacy to provide another type of guarantee.

However, \citet{zhu2019deepleakagegradients} revealed a significant vulnerability in FL by demonstrating how training data can be reconstructed from shared gradients using the DLG algorithm. This attack was refined through subsequent works, such as iDLG \citep{zhao2020idlgimproveddeepleakage} and Inverting Gradients \citep{geiping2020invertinggradientseasy}. More advanced techniques, including GradInversion \citep{yin2021gradientsimagebatchrecovery} and CAFE \citep{jin2022cafecatastrophicdataleakage}, further enhance reconstruction quality but often rely on additional information or specific model architectures. More works featuring attacks include \citet{pmlr-v206-wang23g,NEURIPS2021_fa84632d, chen2021understandingtrainingdataleakagegradients}.

In response to these attacks, several defense strategies have been proposed. One line of methods perturb the gradients shared to the server\citep{9578192,NEURIPS2021_91cff01a}, while another line of work like InstaHide \citep{huang2021instahideinstancehidingschemesprivate}, mixup \citep{DBLP:conf/iclr/ZhangCDL18}, pixelization \citep{fan18} focus on directly protecting the data instead of the gradients. More details about attacks and defenses could be found in \citet{ijcai2022p791,9411833,9987657}. To consolidate research in this area, \citet{liu2024datareconstructionattacksdefenses} proposed a framework to systematically analyze the effectiveness of different attacks and defenses. 

In a similar vein, \citet{fay2023adaptive} explored hyperparameter selection to optimize the privacy-utility trade-off in DP-SGD, \citet{10262055} proposed DP-SGD with adaptive noise. However, these approaches do not account for the local parameter landscape, which we address in our work. 

The lower bound on the reconstruction error was first introduced in previous work \citep{liu2024datareconstructionattacksdefenses} except that only a local approximation was used.

\input{tex/background.tex}

\input{tex/theory.tex}
\input{tex/methodology.tex}
\input{tex/experiments.tex}

\section{Discussion}
In this work, we derived a theoretical reconstruction lower bound and used it to formulate optimal defense methods as improvements of gradient noise and gradient pruning. Our experimental results on MNIST and CIFAR-10 demonstrate the effectiveness of our approach.

As we are only presenting the possibility of a better privacy-utility tradeoff, a key limitation of our methods is the high computational cost of our algorithm. This could be mitigated through simplifications (e.g. layer-wise defense) or lowering the frequency of updating defense parameters. Additionally, the reconstruction bound used in our analysis is not tight. The utilization of more precise bounds or indices that integrate current attack methods remains an open challenge for further research.

Furthermore, our analysis could potentially be applied to other defense methods or designed against other types of attacks. This also remains an open challenge.

\subsection*{Acknowledgements}
QL was partially supported by the NYU Research Catalyst Prize and the Department of Energy under ASCR Award DE-SC0024721.
\bibliography{references}

\newpage
\onecolumn
\appendix
\input{tex/appendix}

\end{document}

%% file: math_commands.tex
\usepackage{amsmath,amsfonts,bm}

\def\EE{\mathbb{E}}
\def\RR{\mathbb{R}}

\DeclareMathOperator{\tr}{tr}

\def\mJ{\mathbf{J}}

\def\vx{\mathbf{x}}
\def\vy{\mathbf{y}}
\def\prune{\text{prune}}
\def\noise{\text{noise}}
\def\dpsgd{\text{DPSGD}}
\newcommand{\norm}[1]{\left\|#1\right\|}
\def\dP{\mathcal{P}}

\def\eqref#1{equation~\ref{#1}}

\def\1{\bm{1}}

\def\vepsilon{{\bm{\epsilon}}}

\def\vx{{\bm{x}}}
\def\vy{{\bm{y}}}

\def\evx{{x}}

\def\mI{{\bm{I}}}
\def\mJ{{\bm{J}}}

\def\mV{{\bm{V}}}

\def\mSigma{{\bm{\Sigma}}}

\DeclareMathAlphabet{\mathsfit}{\encodingdefault}{\sfdefault}{m}{sl}
\SetMathAlphabet{\mathsfit}{bold}{\encodingdefault}{\sfdefault}{bx}{n}

\def\sA{{\mathbb{A}}}

\def\emSigma{{\Sigma}}

\newcommand{\E}{\mathbb{E}}

\newcommand{\R}{\mathbb{R}}

\newcommand{\Var}{\mathrm{Var}}



%% file: tex/background.tex
\section{Preliminaries}
\label{sec:background}

\paragraph{Notations.} 
We denote $\vx\in\R^m$ as the training data generated from a distribution $\dD$. $L(\cdot,\Theta):\R^m\to\R$ is the loss function parameterized by $\Theta\in\R^d$. The model gradient for $\vx$ is $g_\Theta(\vx)\coloneq \nabla_{\vx} L_\Theta(\vx)$. When no ambiguity, we write $g(\vx)$ for brevity. $\vy$ is an (random) observation generated from $\vx$: $\vy=S(g(\vx))$, where $S$ is a random mechanism such as adding noise. 
Let \( \mathcal{P}(S) \) denote the family of distributions over a set \( S \). 

\subsection{Federated Learning and Gradient Reconstruction Attacks}
Different from traditional centralized optimization where we train a model on curated datasets, federated learning (FL) collaboratively trains a model while the data remains decentralized and stored locally on clients. This setup intends to protect users' sensitive data without directly sharing them. 

In FL, each client \( u_i \in \{u_1, \dots, u_n\} \) owns a private dataset \( D_i \), and the global dataset is \( D = \bigcup_{i=1}^n D_i \). A central server aims to train a model \( \Theta \) by solving the optimization problem:
\[
\min_\Theta \sum_{i=1}^n \sum_{\vx_j \in D_i} {L}(\vx_j, \Theta).
\]
During training, stochastic gradient descent (SGD) is conducted, where a subset of (well-connected and active) clients \( U \subset \{1, \dots, n\} \) will interact with the global server: Each active client \( i \in U \) uses a subset \( D'_i \subset D_i \) to create a minibatch \( B = \bigcup_{i \in U} D'_i \). The global minibatch gradient \( \nabla_\Theta {L}(B, \Theta) \) is computed as a weighted average of the individual client gradients:
\[
\nabla_\Theta {L}(B, \Theta) = \frac{1}{|B|} \sum_{i \in U} |D'_i| \nabla_\Theta {L}(D'_i, \Theta^t).
\]
Each client shares \( \langle |D'_i|, \nabla_\Theta {L}(D'_i, \Theta^t) \rangle \) with the server, which then updates the model parameters as:
\[
\Theta^{t+1} \leftarrow \Theta^t - \eta \nabla_\Theta {L}(B, \Theta).
\]

Although these shared gradients contain less information than the raw data, there remains a risk of data leakage, as demonstrated by increasing attention recently~\cite{yin2021gradientsimagebatchrecovery,huang2021instahideinstancehidingschemesprivate,geiping2020invertinggradientseasy}. This work focuses on defending local gradients while minimizing the impact on training utility.

\subsection{The Bayesian Cramér-Rao Lower Bound}
\label{int:bcrb}
The data reconstruction problem is essentially the problem of estimation from random observations. Let $\vx \in \mathbb{R}^d$ represent training data drawn from a distribution $\dD$, $\vy\in \mathbb{R}^K$ denote random observations generated from $\vx$, and $\hat{\vx}(\vy)$ be an estimator of $\vx$. We will introduce the Bayesian Cramer-Rao lower bound that relates to the lowest possible estimation error $\E[\|\hat\vx-\vx\|^2].$ First, assume the following regularity conditions hold \citep{Crafts_2024,citeulike:10401058}:

\textbf{Assumption 1 (Support).} The support of $\dD$ is either $\mathbb{R}^d$ or an open bounded subset of $\mathbb{R}^d$ with a piecewise smooth boundary.

\textbf{Assumption 2 (Existence of Derivatives).} The derivatives $\left[ \nabla_{\vx} p(\vx, \vy) \right]_i$ for $i = 1, \ldots, d$, exist and are absolutely integrable.

\textbf{Assumption 3 (Finite Bias).} The bias, defined as 
\[
B(\vx) \coloneq \int (\hat{\vx}(\vy) - \vx) p(\vy \mid \vx) \, d\vy,
\]
is finite for all $\vx$.

\textbf{Assumption 4 (Exchanging Derivative and Integral).} The probability function $p(\vx, \vy)$ and estimator $\hat{\vx}(\vy)$ satisfy:
\begin{align*}
&\nabla_{\vx} \int p(\vx, \vy) \left[ \hat{\vx}(\vy) - \vx \right]^T d\vy \\
= &\int \nabla_{\vx} \left( p(\vx, \vy) \left[ \hat{\vx}(\vy) - \vx \right]^T \right) d\vy
\end{align*}
for all $\vx$.

\textbf{Assumption 5 (Error Boundary Conditions).} For any point $\vx$ on the boundary of $\text{supp}(\dD)$, and any sequence $\{\vx_i\}_{i=0}^{\infty}$ such that $\vx_i \in \text{supp}(\dD)$ and $\vx_i \to \vx$, we have $B(\vx_i) p(\vx_i) \to 0$.

These assumptions are satisfied by a wide range of setups. For image classification, the dataset has bounded support and the defense a differentiable density function $p(\vx, \vy)$. When we add a small Gaussian noise to the training data, all Assumptions 1 to 5 hold.  \citep{Crafts_2024}

Given these assumptions, the Bayesian Cramér-Rao Lower Bound is as follows:
\[
\mathbb{E}_{\vx, \vy} \left[ (\hat{\vx}(\vy) - \vx) (\hat{\vx}(\vy) - \vx)^T \right] \succeq \mV_B \coloneq \mJ_B^{-1};
\]
where $\mJ_B \in \mathbb{R}^{D \times D}$ is the Bayesian information matrix:
\[
\mJ_B \coloneq \mathbb{E}_{\vx, \vy} \left[ \nabla_{\vx} \log p(\vx, \vy) \nabla_{\vx} \log p(\vx, \vy)^T \right].
\]
The matrix $\mJ_B$ can be decomposed into two components: 
\[
\mJ_B = \mJ_P + \mJ_D;
\]
where $\mJ_P$ is the prior-informed term:
\[
\mJ_P \coloneq \mathbb{E}_{\vx} \left[ \nabla_{\vx} \log p(\vx) \nabla_{\vx} \log p(\vx)^T \right];
\]
and $\mJ_D$ is the data-informed term:
\[
\mJ_D \coloneq \mathbb{E}_{\vx} \left[ \mJ_F(\vx) \right].
\]
Here, $\mJ_F(\vx)$ represents the Fisher information matrix:
\[
\mJ_F(\vx) \coloneq \mathbb{E}_{\vy \mid \vx} \left[ \nabla_{\vx} \log p(\vy \mid \vx) \nabla_{\vx} \log p(\vy \mid \vx)^T \right].
\]

%% file: tex/theory.tex
\section{Methodology}
\label{sec:theory}
To optimize the trade-off between the reconstruction error lower bound and training utility, we treat each observed coordinate separately and design defending strategies customized to the current data batch and model parameters, instead of a universal strategy like a constant noise level in DP-SGD. We will first present our derivation of the reconstruction error lower bound and our definition of the training utility. Then, we introduce an optimization objective to find the optimal defense parameters (such as noise's covariance matrix) that balance reconstruction error and utility. 
\subsection{The Reconstruction Error Lower Bound}
To prevent data leakage, our goal is to maximize the lower bound of the reconstruction error among all estimators (reconstruction algorithms). 

For a randomized defense mechanism $S:\R^d\to\mathcal{P}(\mathbb{R}^d)$ (e.g., adding noise to the gradients), the defended gradients are $\vy \sim S(g(\vx))$. For any reconstruction algorithm $R:\R^d\to\R^m$, the expected reconstruction error against the defense is:
$$\E_{\vx\sim\dD}\E_{\vy\sim S(g(\vx))}\norm{R(\vy)-\vx}^2.$$

\begin{definition}
\label{def:lowerbound}
For a data distribution $\dD\in\mathcal{P}(\mathbb{R}^m)$, a gradient function $g:\R^m\to\R^d$, and a defense mechanism $S:\R^d\to\mathcal{P}(\mathbb{R}^d)$, the reconstruction error lower bound $B_{\dD,S}$ is as follows:
\[B_{\dD,S}\coloneq\min_{R:\R^d\to\R^m}\E_{\vx\sim\dD}\E_{\vy\sim S(g(\vx))}\norm{R(\vy)-\vx}^2.\]
\end{definition}

We utilize the Bayesian C-R lower bound to lower bound the reconstruction error lower bound:

\begin{theorem}
\label{thm:lowerbound}
Let $B_{\dD,S}$ be as defined in Definition \ref{def:lowerbound}. Under Assumptions 1 to 5, we lower bound $B_{\dD,S}$ by:
\begin{equation}
\label{def:reconsbound}
B_{\dD,S}\ge\frac{d^2}{\EE_{x\sim\dD}\left[\tr(\mJ_F(x))\right]+d\cdot\lambda_1(\mJ_P)},
\end{equation}
where $\mJ_F(\vx)$ is given by:
\[
\mJ_F(\vx) \coloneq \mathbb{E}_{\vy\sim S(g(\vx))} \left[ \nabla_{\vx} \log p_{S(\vx)}(\vy) \nabla_{\vx} \log p_{S(\vx)}(\vy)^\top \right]
\]
and \( \mJ_P \) by:
\[
\mJ_P \coloneq \mathbb{E}_{\vx} \left[ \nabla_{\vx} \log p_\dD(\vx) \nabla_{\vx} \log p_\dD(\vx)^\top \right].
\]
\end{theorem}

Here, $\mJ_F(\vx)$ is data-informed and depends on the defense method $S$, while \( \mJ_P \) is prior-informed and depends only on the distribution $\dD$. If the prior is flat, $\lambda_1(\mJ_P)\approx 0$.

\begin{remark}
The lower bound decreases with $\tr(\mJ_F(x))$. Thus, to improve our reconstruction error lower bound, we minimize $\tr(\mJ_F(x))$ for our defense method.
\end{remark}

For mixed defense methods, we derive a bound similar to Eq. \ref{def:reconsbound}. Let $S$ be defined as $Q(g(\vx),i)$, where $i$ is an identifier sampled from distribution $\mathcal{I}$. For each $i$, $Q(\cdot, i)$ satisfies Assumptions 1 to 5 independently, representing a unique defense mechanism.

Using Jensen's inequality, we obtain the lower bound for mixed defense:
\begin{align}
B_{\dD,\mathcal{I}}=&\E_{i\sim\mathcal{I}}B_{\dD,Q(\cdot,i)}\notag\\
\label{def:mixedlowerbound}
\ge&\frac{d^2}{\E_{i\sim\mathcal{I}}\EE_{x\sim\dD}\left[\tr(\mJ_{F,Q(\cdot,i)}(x))\right]+d\cdot\lambda_1(\mJ_P)},
\end{align}
where $\mJ_{F,Q(\cdot,i)}(\vx)$ represents
\[
\mathbb{E}_{\vy\sim Q(\vx,i)} \left[ \nabla_{\vx} \log p_{S(\vx)}(\vy) \nabla_{\vx} \log p_{S(\vx)}(\vy)^\top \right].
\]

\subsection{Training Utility}
To assess utility, we analyze the model loss after one step of gradient descent update. Due to the complexity of the loss landscape, we make an approximation by the first-order Taylor expansion. The second-order Taylor approximation might seem straightforward and accurate, but the result is only meaningful when the loss function is convex with respect to the model parameters.\footnote{Otherwise, the approximation will suggest that larger noise increases utility, leading to an unrealistic result of infinitely large optimal noise.} \citet{fay2023adaptive} analyzed the utility of DP-SGD by using the lower bound of the expected loss, derived by assuming the loss function M-smooth. However, this oversimplifies the loss landscape by using the same isotropic convex function regardless of training data or model parameters. Optimizing this bound also requires choosing the optimal learning rate, while we aim to separate the defense method from the learning rate to make our defense more general.

To avoid these issues, we use the expectation and variance of the model loss after one gradient update, approximated by the first-order Taylor expansion, as our utility measure. These measures 1) are independent of the learning rate; and 2) contain information about the loss function's landscape. A good defense method should minimally impact training utility, therefore we maximize the expectation of the training loss and minimize its variance.

\begin{definition}
Given training data $\vx\in\R^m$ from distribution $\dD$, a model with $d$ parameters $\Theta$, and a loss function $L:\R^m\times\R^d\to\R$, the first-order utility of a defense method $S$ on $\vx$ is the expected decrease in loss after one gradient update:
\begin{equation}
\label{eq:firstorderloss}
U_1(S,\Theta)=\E_{\vx\sim\dD}\E_{\vy\sim S(\nabla_\vx L(\vx,\Theta))}\nabla_\vx L(\vx,\Theta)\cdot \vy.
\end{equation}
The second-order utility is defined as the negative variance of the loss after the update:
\begin{equation}
\label{eq:secondorderloss}
U_2(S,\Theta)=-\E_{\vx\sim\dD}\Var_{\vy\sim S(\nabla_\vx L(\vx,\Theta))}\nabla_\vx L(\vx,\Theta)\cdot \vy.
\end{equation}
\end{definition}
We optimize the first-order utility first, and when the first-order utility is constant, we compare the second-order utility.

\subsection{Optimal Gradient Noise}
\label{sec:optimalnoise}

\paragraph{Gradient Noise.} One of the simplest defense methods, also a step in DP-SGD \citep{abadi16deep}, is to add Gaussian noise to the model gradients before sharing. For a given covariance matrix $\mSigma\in\R^{d\times d}$, the gradient noise defense is as follows:
\begin{equation}
\label{def:noise}
S_{\noise,\mSigma}(\vx)=\mathcal{N}(\vx,\mSigma).
\end{equation}

\paragraph{Optimal Gradient Noise.} The first-order utility defined in Eq. \ref{eq:firstorderloss} remains constant regardless of the choice of the covariance matrix $\mSigma$:
\[
U_1(S_{\noise,\mSigma},\Theta) = \E_{\vx\sim\dD}\nabla_\vx L(\vx,\Theta)\cdot \nabla_\vx L(\vx,\Theta)^\top.
\]
Thus, we focus on maximizing the second-order utility. Assuming independent noise across parameters (as in DP-SGD), we limit our analysis to diagonal matrices. In this case, the second-order utility equals:
\begin{equation}
U_2(S_{\noise,\mSigma},\Theta) = -\sum_{i=1}^d \left( \frac{\partial L(\vx,\Theta)}{\partial \vx_i} \right)^2 \emSigma_{i,i}.
\end{equation}
For a higher reconstruction error lower bound, we minimize $\E_{\vx\sim\dD}\tr(\mJ_F(\vx))$, where:
\begin{equation}
\label{eq:noisebound}
\mJ_F(\vx) = \sum_{i=1}^d\frac{\norm{\nabla_{\vx} g(\vx)}^2}{ \emSigma_{i,i}}.
\end{equation}

This decomposition allows us to separate the influence the defense of each parameter has on utility and privacy, setting the stage for deriving the optimal noise.

\begin{theorem}[Optimal Gradient Noise]
\label{thm:optimalnoise}
Under assumptions 1 to 5, and assuming $\E_{\vx\sim\dD}g_i(\vx)^2 > 0$ for all $i$, \footnote{Special cases where certain entries of $\E_{\vx\sim\dD}g_i(\vx)^2$ are zero (e.g., in models with ReLU activation) are discussed in the appendix.}
the optimal noise matrix $\mSigma$ for a given utility budget $U_2(S_{\noise,\mSigma},\Theta)\ge -C$ has diagonal elements:
\[
\emSigma_{i,i} = \lambda \sqrt{\frac{\E_{\vx\sim\dD}\norm{\nabla_\vx g_i(\vx)}^2}{ \E_{\vx\sim\dD}g_i(\vx)^2}},
\]
where $\lambda$ is a constant, and $g_i$ is the $i$-th component of the gradient $g(\vx)= \nabla_{\vx} L_\Theta(\vx)$.
\end{theorem}

In the special case where $\dD$ is supported on a small neighborhood of $\vx$, corresponding to the attacker having an approximation of the data, we could approximate and simplify the locally optimal noise by using the value at $\vx$ to replace the expectations:
\begin{equation}
\label{eq:localnoise}
\emSigma_{i,i}(\vx) = \lambda(\vx)\frac{\norm{\nabla_\vx g_i(\vx)}}{|g_i(\vx)|}.
\end{equation}

\subsection{Optimal DP-SGD}

\paragraph{DP-SGD.} We extend our analysis in Section \ref{sec:optimalnoise} to optimize the noise in DP-SGD. DP-SGD differs from gradient noise by a gradient clipping step before adding noise. For a fixed clipping threshold $P$, let $g_P(\vx)$ represent the clipped gradients with elements:
\begin{equation}
g_P(\vx)_i=
\begin{cases}
g(\vx)_i, &\text{if } -P < g(\vx)_i < P, \\
P, &\text{if } g(\vx)_i \ge P, \\
-P, &\text{if } g(\vx)_i \le -P.
\end{cases}
\end{equation}

The generalized DP-SGD defense, $S_{\dpsgd,\mSigma,P}:\R^d\to\dP(\R^d)$, is as follows:
\begin{equation}
\label{def:dpsgd}
S_{\dpsgd,\mSigma,P}(\vx)=\mathcal{N}(g_P(\vx),\mSigma),
\end{equation}
where $\mSigma$ is the noise covariance. For $\mSigma = \epsilon \mI_d$, this reduces to standard DP-SGD. Since the first-order utility is constant, we find the $\mSigma$ that optimizes the second-order utility:
\begin{equation}
\label{eq:noiseutil}
U_2(S_{\dpsgd,\mSigma,P},\Theta)=-\sum_{i=1}^d \left(\frac{\partial L(\vx,\Theta)}{\partial \vx_i}\right)^2 \Sigma_{i,i}.
\end{equation}

\begin{theorem}[Optimal DP-SGD]
\label{thm:optimaldpsgd}
Under assumptions 1 to 5, and assuming $\E_{\vx\sim\dD} g_i(\vx)^2 > 0$ for all $i$, the optimal noise matrix $\mSigma$ that maximizes the reconstruction error lower bound (Eq. \ref{def:reconsbound}) for a utility budget of 
\[
U_2(S_{\dpsgd,\mSigma,P},\Theta) \ge -C
\]
has diagonal elements:
\begin{equation}
\Sigma_{i,i} = \lambda \sqrt{\frac{\E_{\vx\sim\dD} \norm{\nabla_\vx g_P(\vx)_i}^2}{\E_{\vx\sim\dD} g_P(\vx)_i^2}},
\end{equation}
where $\lambda$ is a constant, and $g(\vx) = \nabla_{\vx} L_\Theta(\vx)$ is the model gradient on the data $\vx$.
\end{theorem}

In the special case where $\dD$ is supported on a small neighborhood of $\vx$, the locally optimal DP-SGD noise becomes:
\begin{equation}
\label{eq:localdpsgd}
\Sigma_{i,i}=
\begin{cases}
\lambda\sqrt{\frac{\E_{\vx\sim\dD} \norm{\nabla_\vx g_i(\vx)}^2}{\E_{\vx\sim\dD} g_i(\vx)^2}}, & \text{if } |g_i(\vx)| < P, \\
0, & \text{if } |g_i(\vx)| = P.
\end{cases}
\end{equation}
To summarize, we could optimize the privacy-utility trade-off for DP-SGD by changing the noise to zero for clipped gradients, and to our optimal noise for other gradients.
\subsection{Optimal Gradient Pruning}

\paragraph{Gradient Pruning.} Gradient pruning reduces the number of parameters in the shared gradient by zeroing out less significant gradients during training. Inspired by gradient compression \citep{lin2017deep,tsuzuku2018variance}, this approach prunes gradients with the smallest magnitude \citep{zhu2019deepleakagegradients}. It is also the most effective defense against DLG~\citep{zhu2019deepleakagegradients}.

For a given set of parameters $\sA$, the gradient pruning defense method $S_{\prune,\sA}:\R^d\to\dP(\R^d)$ is as follows:
\begin{equation}
\label{def:prune}
S_{\prune,\sA}(\vx)_i=
\begin{cases}
0, & \text{if } i \in \sA, \\
\vx_i, & \text{if } i \notin \sA.
\end{cases}
\end{equation}
\paragraph{Optimal Gradient Pruning} Under assumptions 1 to 5, the first-order utility of gradient pruning equals the sum of the squared unpruned gradients:
\begin{equation}
\label{eq:utilityprune}
U_1(S,\Theta) = \sum_{i \notin \sA} \left( \frac{\partial L(\vx,\Theta)}{\partial \vx_i} \right)^2.
\end{equation}

Since gradient pruning introduces no randomness, an accurate reconstruction is theoretically possible when the number of unpruned parameters exceeds the input dimension. To address this problem, we add a small noise to the unpruned gradients and analyze the noisy version of gradient pruning:
\begin{equation}
\label{def:noisyprune}
S_{\prune,\sA,\mSigma}(\vx)_i=
\begin{cases}
0, & \text{if } i \in \sA, \\
\mathcal{N}(\vx_i, \mSigma_{i,i}), & \text{if } i \notin \sA.
\end{cases}
\end{equation}
For $\mSigma = \epsilon \mSigma_0$, this collapses to the original pruning method when $\mSigma_0$ remains constant and $\epsilon\to 0$.

\begin{theorem}[Optimal Gradient Pruning]
\label{thm:optimalprune}
Under assumptions 1 to 5, and assuming $\E_{\vx \sim \dD} g_i(\vx)^2 > 0$ for all $i$, the optimal pruning distribution $\mathcal{R}$ for generating $\sA$, given the utility budget 
\[
\E_{\sA\sim\mathcal{R}}U_1(S_{\prune,\sA,\mSigma},\Theta) \geq C,
\]
prunes elements with the smallest value of:
\[
k_i = \frac{\E_{\vx \sim \dD} \norm{\nabla_\vx g_i(\vx)}^2}{\E_{\vx \sim \dD} g_i(\vx)^2},
\]
where $g(\vx) = \nabla_{\vx} L_\Theta(\vx)$ is the model gradient, and $g_i$ represents its $i$-th component.
\end{theorem}

\begin{remark}
When a deterministic set does not match the utility budget, parameters on the borderline are pruned with positive probability. When this happens, the optimal defense is a mixed defense.
\end{remark}

Similar to previous sections, we derive locally optimal gradient pruning, which prunes parameters $i$ with the smallest value $k_i'$:
\begin{equation}
\label{eq:localprune}
k'_i = \frac{\norm{\nabla_\vx g_i(\vx)}}{g_i(\vx)}.
\end{equation}
An additional feature of the locally optimal version is that it is also the optimal defense when using optimal noise instead of identical noise in Theorem \ref{thm:optimalnoise}.

%% file: tex/methodology.tex
\subsection{Algorithm design}
\label{sec:methodology}

Because of the high computational cost of the expectation terms in the globally optimal defense methods (Theorems \ref{thm:optimalnoise} to \ref{thm:optimalprune}), our implementations are based on the locally optimal versions (Eqs. \ref{eq:localnoise}, \ref{eq:localdpsgd} and \ref{eq:localprune}).

When computing optimal defense parameters, calculating the Jacobian matrix of model gradients on input data is especially challenging. The full Jacobian matrix for an image with a resolution of $32 \times 32$ would require roughly 3000 times the memory of the model itself, which is prohibitively large. We resolve this problem by using the forward differentiation method to save computational cost and use approximation to save memory cost. 

The forward method \citep{10.5555/1455489} tracks gradients based on the input tensor size rather than the output tensor size, and therefore more efficient since we are dealing with low input and high output dimensions. We approximate the $l_2$-norm of the gradients using Lemma \ref{lem:randomsketching}:

\begin{lemma}
\label{lem:randomsketching}
Given a differentiable function $f:\R^d\to\R$ and a constant $\epsilon>0$. For any number of samples $k\in \mathbb{Z}$ and random vectors $\vx_1,\dots,\vx_k$ sampled from $\mathcal{N}(0,\mI_n)$, we have that
\begin{equation}
\label{eq:l2approx}
\left|\norm{\nabla_\vx f(\vx)}^2-\frac{1}{k} \sum_{j=1}^k \norm{\frac{\partial f_i(\vx + \alpha \vx_j)}{\partial\alpha}}^2_{\alpha=0}\right|\le \epsilon \norm{\nabla_\vx f(\vx)}^2
\end{equation}
with probability at least $1-\frac{2}{k\epsilon^2}$ for any $\vx\in\R^d$.
\end{lemma}

This allows us to approximate the $l_2$-norms without the entire Jacobian matrix, significantly reducing computational and memory cost. Our resulting algorithm is outlined in Algorithm \ref{alg:approxnoiseprune}.

\begin{algorithm}[h]
\caption{Approximate Locally Optimal Gradient Noise}
\label{alg:approxnoiseprune}
\begin{algorithmic}[1]
\STATE \textbf{Input:} Model parameters $\Theta \in \mathbb{R}^{d}$, loss function $L(\vx,\Theta)$, number of samples $k$, noise scale $\lambda$, small constant $c$
\STATE \textbf{Output:} Defended model gradients $S(g(\vx))$
\STATE \textbf{Step 1:} Compute the model gradients $g(\vx) = \nabla_{\Theta}L(\vx,\Theta)$
\STATE \textbf{Step 2:} Sample $k$ random vectors $\vx_1, \dots, \vx_k \sim \mathcal{N}(0, \mI_{n})$
\STATE \textbf{Step 3:} Approximate $\norm{\nabla_\vx g_i(\vx)}^2$ with 
$\frac{1}{k} \sum_{j=1}^k \norm{\frac{\partial g_i(\vx + \alpha \vx_j)}{\partial\alpha}}^2_{\alpha=0}.$
\STATE \textbf{Step 4:} Compute the optimal noise matrix:
\begin{equation}
\emSigma_{i,i} = \lambda \frac{\norm{\nabla_\vx g_i(\vx)}^2}{\max(|g_i(\vx)|, c)}
\end{equation}
\STATE \textbf{Step 5:} Sample $\vepsilon \sim \mathcal{N}(0,\emSigma)$ and return the defended gradients $g(\vx) + \vepsilon$
\end{algorithmic}
\end{algorithm}

The optimal pruning algorithm follows the same steps, with differences in Steps 4 and 5, where we calculate the index for pruning (Eq. \ref{eq:localprune}) instead of the noise scale.

%% file: tex/experiments.tex
\section{Experiments}
\label{sec:experiments}
We compare our proposed algorithms with existing defense methods on two datasets: MNIST \citep{MNIST} and CIFAR-10 \citep{krizhevsky2009learning}. As our algorithm employs different defenses on different parameters, we use an attack that treats parameters equally. One attack with such property is the Inverting Gradients attack \citep{geiping2020invertinggradientseasy}, a powerful attack that does not require extra information or specific model architecture. 

\subsection{MNIST}
\vspace{-3pt}
The MNIST dataset consists of $28 \times 28$ grayscale images of handwritten digits, serving as a simple test for our algorithm.

\subsubsection{Gradient Pruning}
\vspace{-3pt}
We apply different pruning thresholds to a randomly initialized Convolutional Neural Network \citep{4082265}, using 4 batches of 16 images to compute gradients. These gradients were defended using gradient pruning and our optimal gradient pruning, followed by an Inverting Gradients attack. Figure \ref{fig:mnistprunek} shows that our method consistently achieves higher Mean Squared Error (MSE) and lower Peak Signal-to-Noise Ratio (PSNR) at commonly used high pruning ratios, indicating stronger defenses.
\begin{figure}
    \centering  \includegraphics[width=\linewidth]{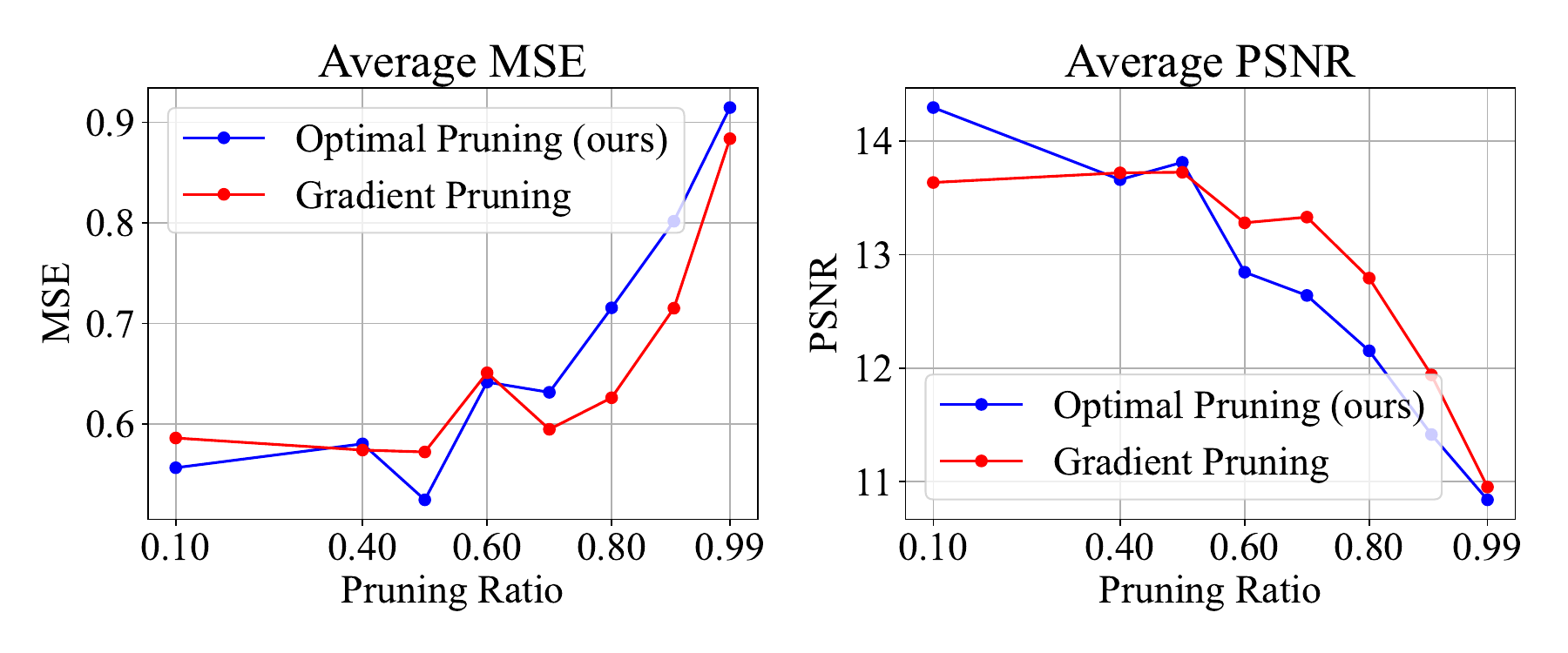}
    \vspace{-20pt}
    \caption{Average reconstruction indexes based on Gradient Inversion with batch size 16.}
    \label{fig:mnistprunek}
\end{figure}

\begin{figure}
    \centering
    \includegraphics[width=1\linewidth]{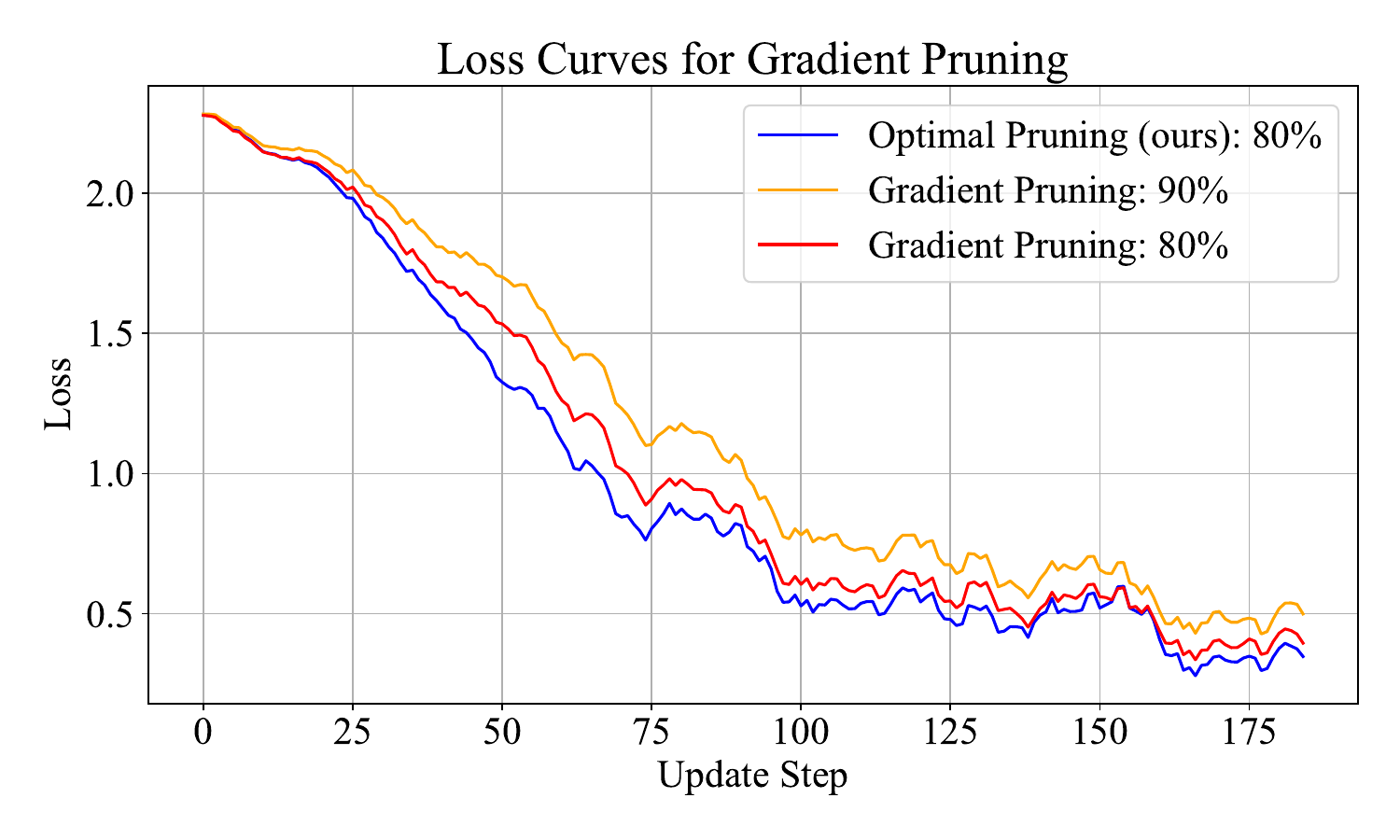}
    \vspace{-20pt}
    \caption{Training curves of CNN on MNIST with 80\% \& 90\% gradient pruning and 80\% optimal pruning (smoothed with window size 8). 80\% optimal pruning outperforms 90\% gradient pruning in training.}
    \label{fig:prunecurves}
\end{figure}

\begin{figure}
    \centering
    \includegraphics[width=1\linewidth]{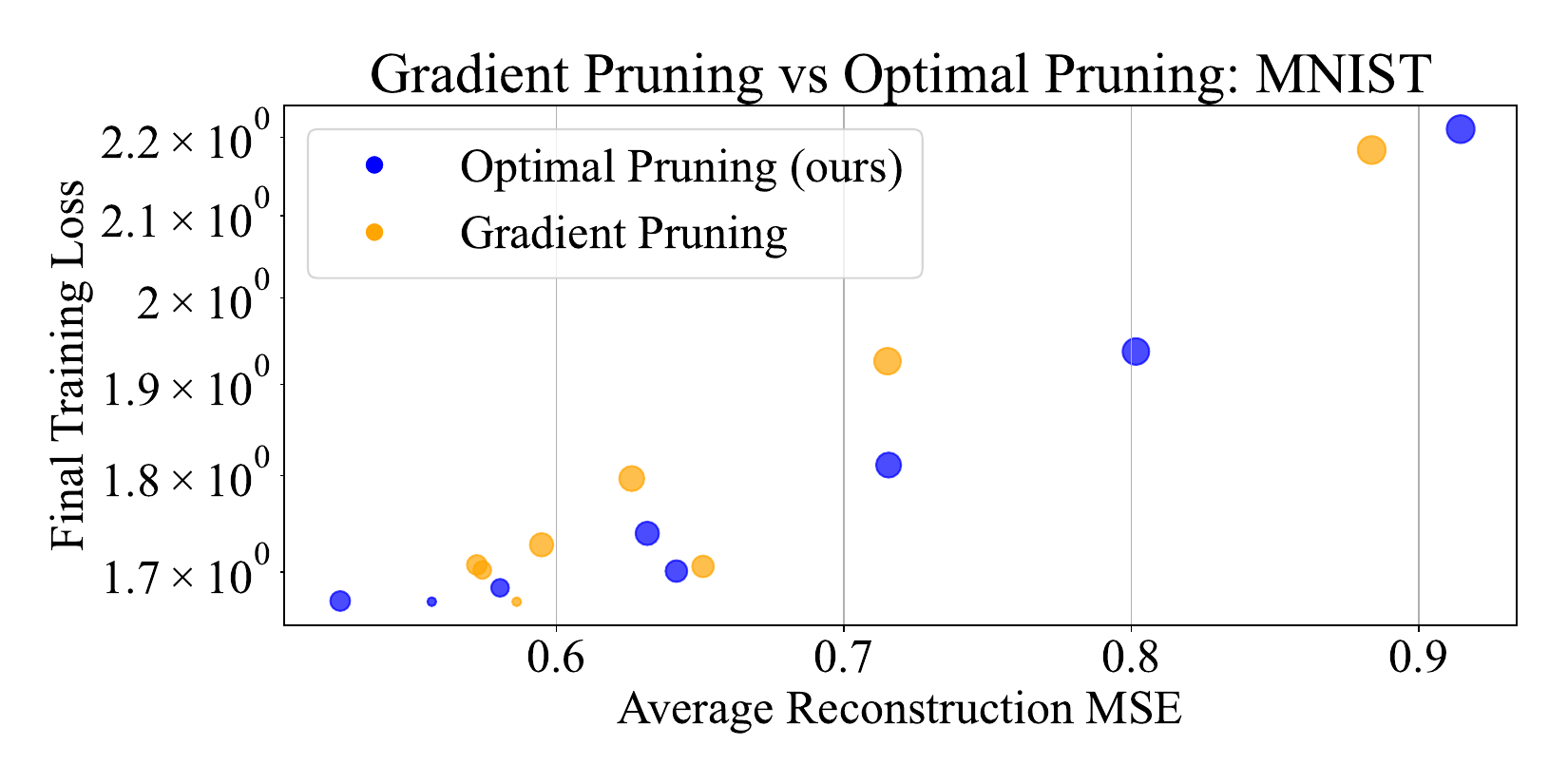}
    \vspace{-25pt}
    \caption{Scatter plot of gradient pruning and our optimal pruning on MNIST. X-axis: average reconstruction MSE. Y axis: Training loss on 64 samples. Size of points: Pruning ratio.}
    \label{fig:mnistscattermse}
\end{figure}

To assess training utility, we trained the models under a federated learning setting. Figure \ref{fig:prunecurves} shows that 80\% optimal pruning outperforms 90\% gradient pruning in training speed, while we showed that they have similar privacy in Figure \ref{fig:mnistprunek}. The scatter plot in Figure \ref{fig:mnistscattermse} shows the privacy-utility trade-off for a wider range of pruning ratios, indicating the superior privacy-utility trade-off of our method.

\vspace{-3pt}
\subsubsection{DP-SGD}
\vspace{-3pt}
\begin{figure}
    \centering    \includegraphics[width=1\linewidth]{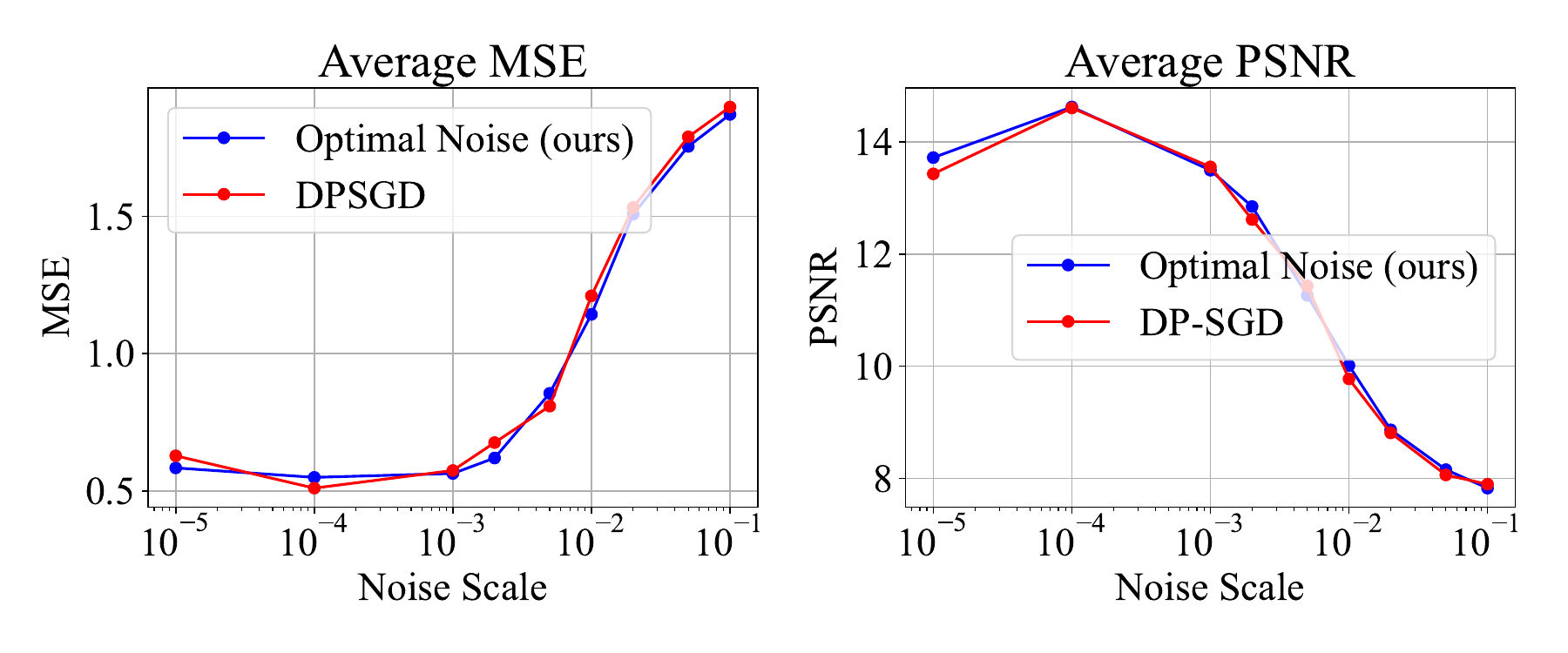}
    \vspace{-20pt}
    \caption{Average reconstruction indexes based on Gradient Inversion for DP-SGD. The noise scale equals the Frobenius norm of the covariance matrix. }
    \label{fig:mnistnoise}
\end{figure}

\begin{figure}
    \centering
    \includegraphics[width=1\linewidth]{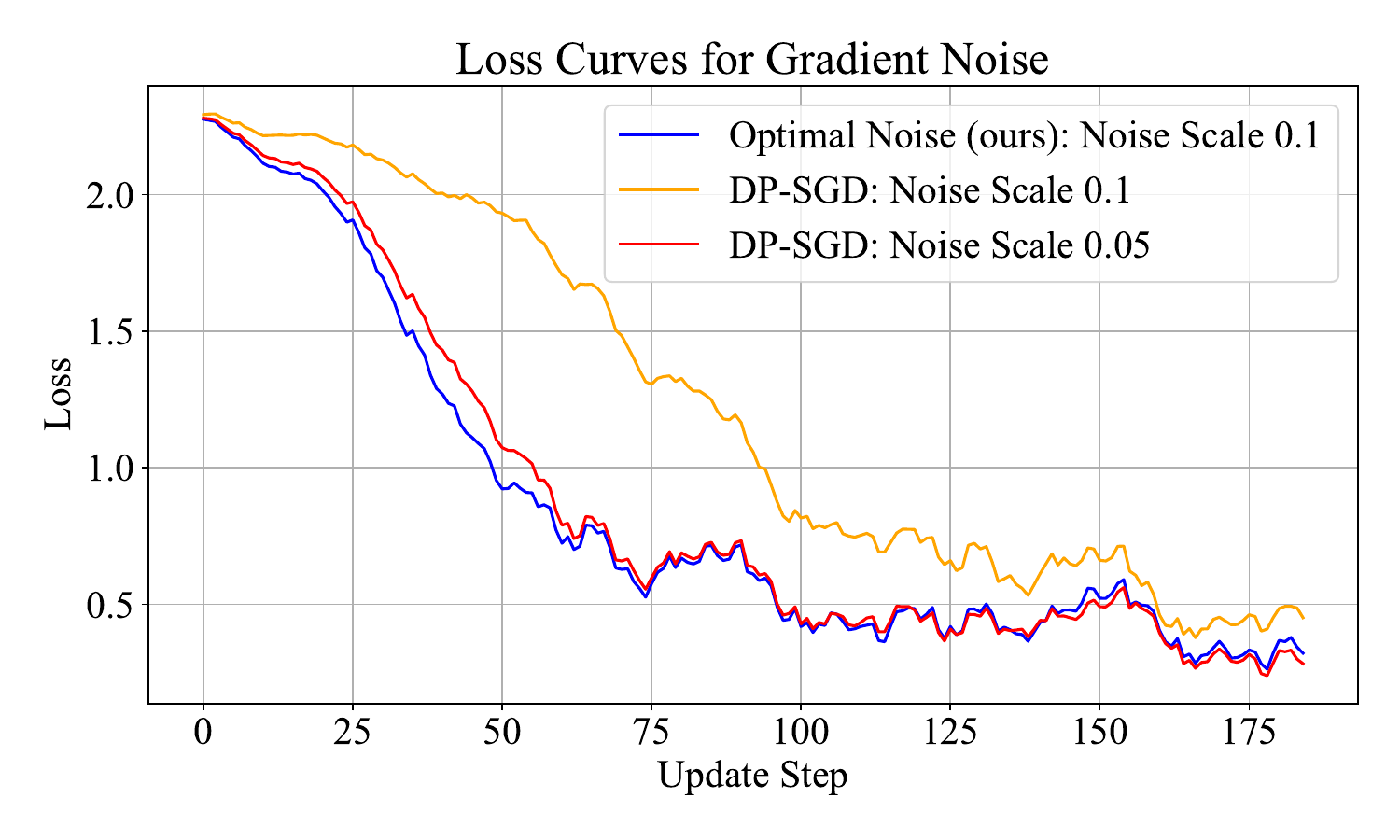}
    \vspace{-20pt}
    \caption{Training curves of CNN on the MNIST dataset.  (Smoothed with window size 8) Optimal noise with a scale of 0.1 outperforms DP-SGD with a scale of 0.1.}
    \label{fig:mnistnoisecurve}
\end{figure}

\begin{figure}
    \centering
    \includegraphics[width=\linewidth]{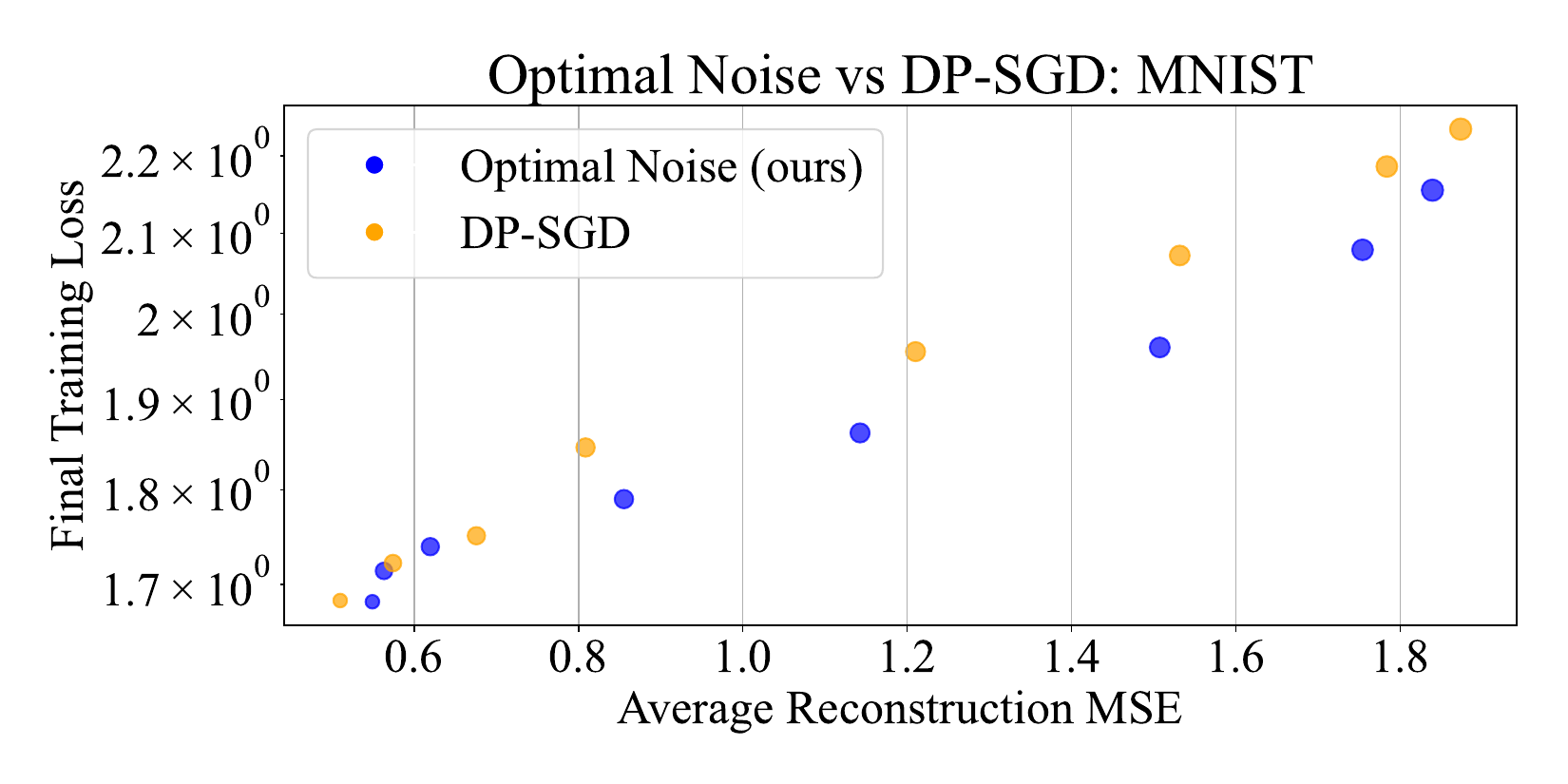}
    \vspace{-20pt}
    \caption{Comparison of optimal noise and DP-SGD on MNIST. X-axis: average reconstruction MSE. Y axis: Training loss on 64 samples.}
    \label{fig:mnistscatterdpsgdmse}
\end{figure}

\begin{figure}
    \centering
    \includegraphics[width=0.6\linewidth]{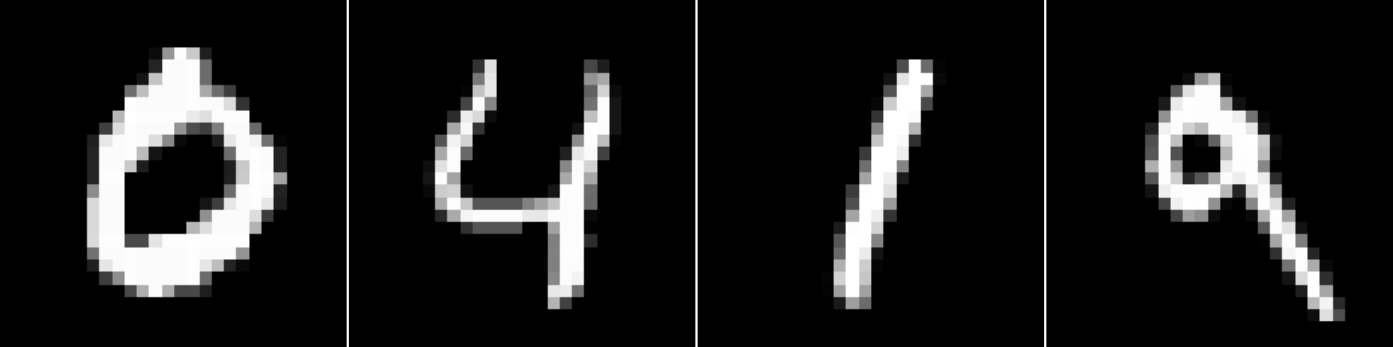}
    \includegraphics[width=0.6\linewidth]{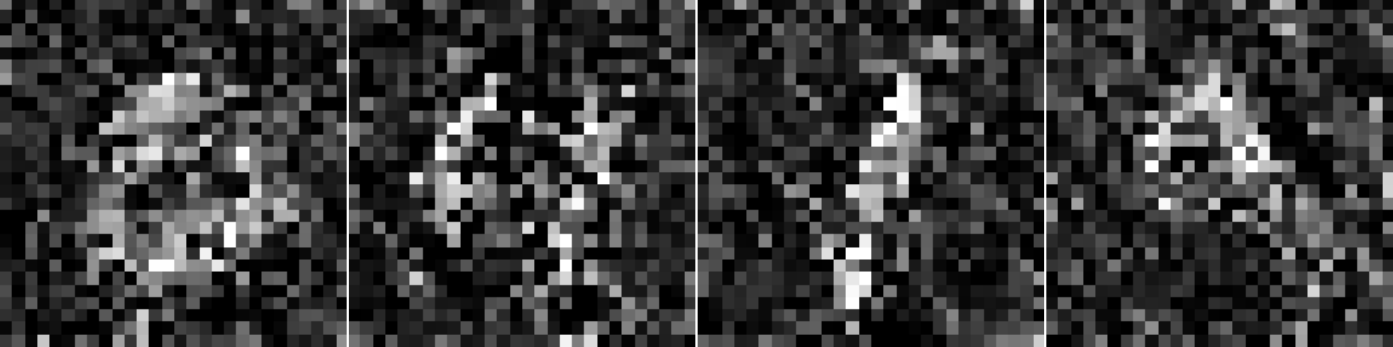}
    \includegraphics[width=0.6\linewidth]{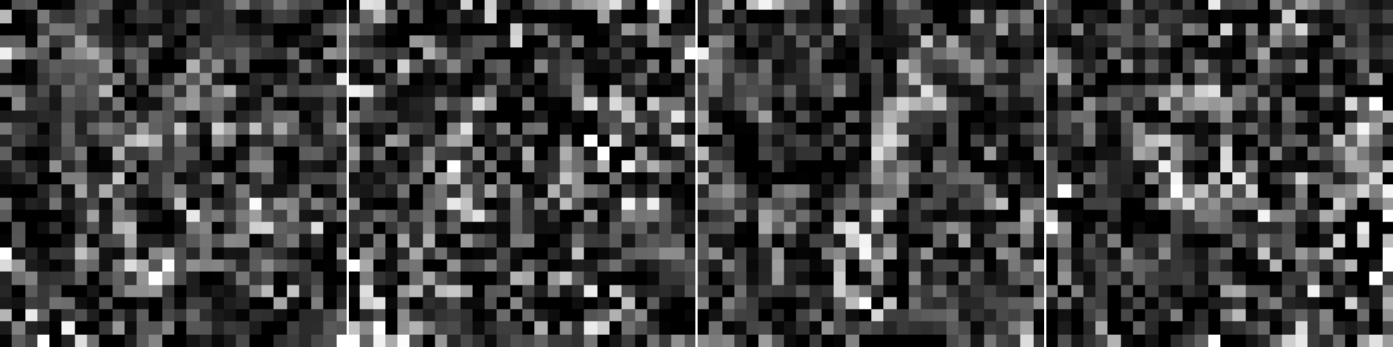}
    \caption{Reconstruction from the MNIST dataset with batch size 4. First row: ground truth. 
    Second row: DP-SGD with scale 0.05. Third row: optimal noise with scale 0.1. Our method has better privacy when the performance in training is similar.}
    \label{fig:visualizenoisemnist}
\end{figure}
We also evaluated our defense on DP-SGD. As shown in Figure \ref{fig:mnistnoise}, our optimal noise achieves comparable performance to DP-SGD in terms of defense at the same noise scale, while our algorithm has faster learning speed (Figure \ref{fig:mnistnoisecurve}). The scatter plot in Figure \ref{fig:mnistscatterdpsgdmse} further demonstrates the improved privacy-utility trade-off of our approach. Visualization of the reconstruction in Figure \ref{fig:visualizenoisemnist} shows better protection against attacks using our optimal noise for the same level of training utility.

\subsection{CIFAR-10}
We extend our experiments to the CIFAR-10 dataset, consisting of colored images with size $32\times 32$. We used a larger model with 2.9M parameters.

\subsubsection{Optimal Pruning}

\begin{figure}[htb]
    \centering
    \includegraphics[width=\linewidth]{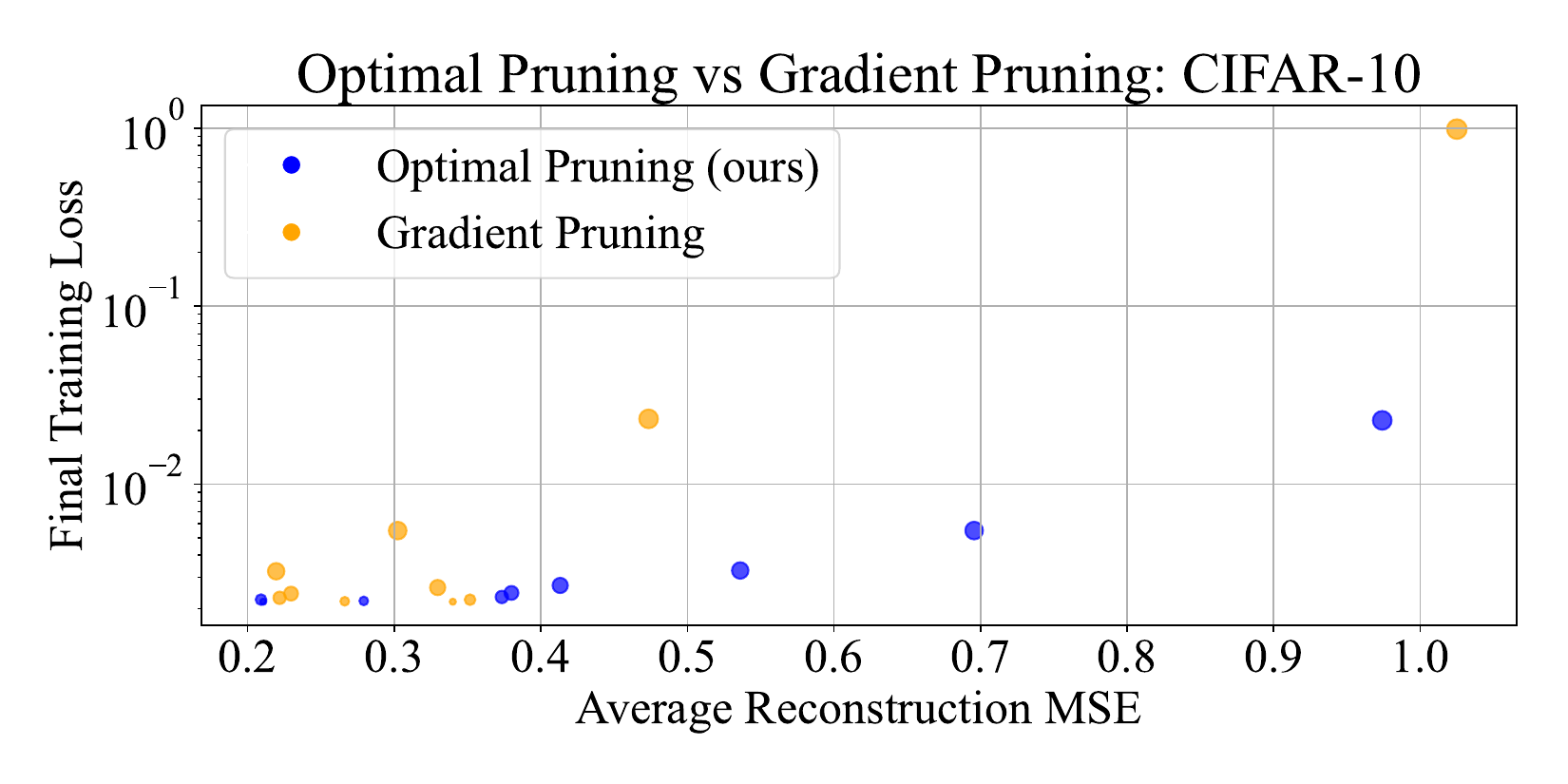}
    \caption{Comparison of optimal pruning and gradient pruning on CIFAR-10. X-axis: average MSE. Y axis: Training loss on 8 samples.}
    \label{fig:cifarprune}
\end{figure}
\begin{figure}[htb]
    \centering
    \includegraphics[width=0.6\linewidth]{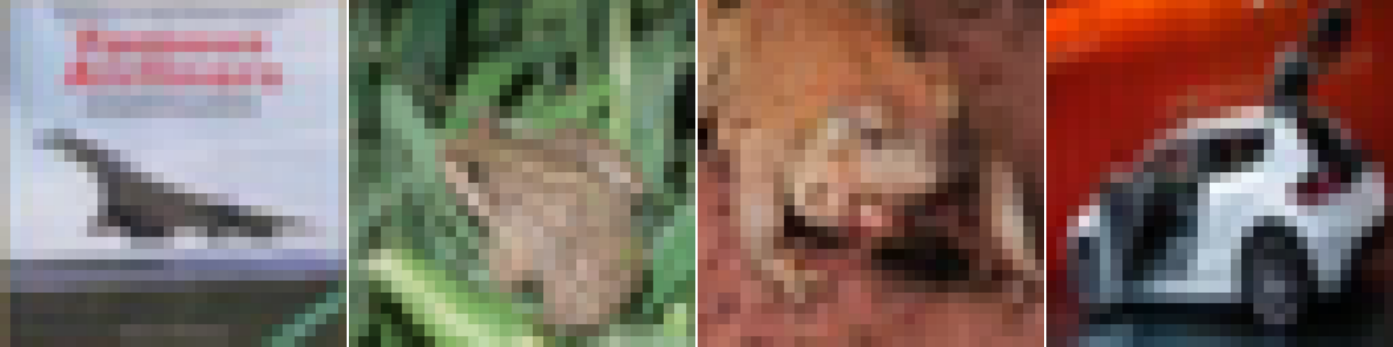}
    \includegraphics[width=0.6\linewidth]{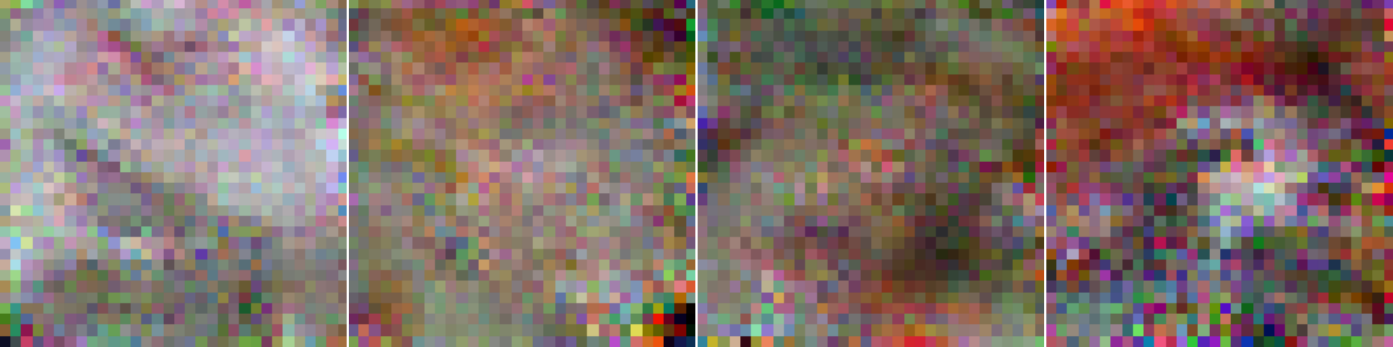}
    \includegraphics[width=0.6\linewidth]{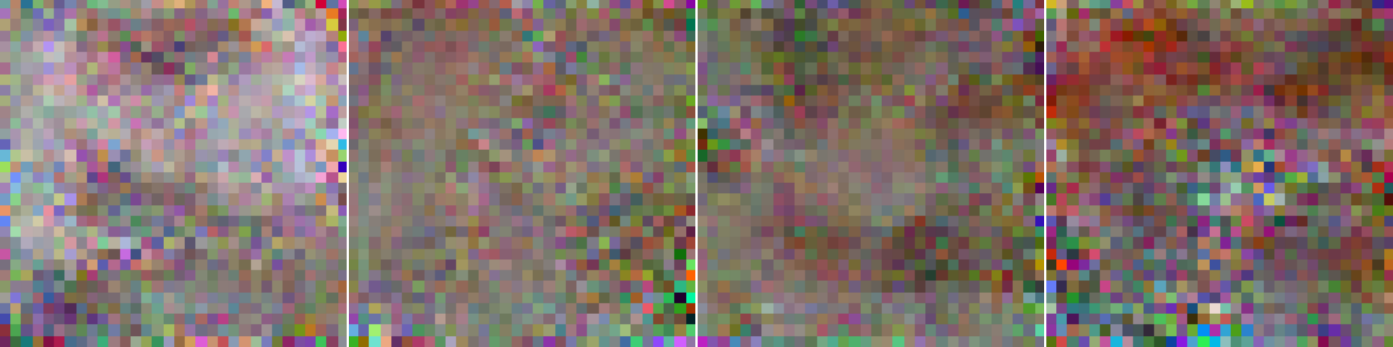}
    \caption{Reconstruction from the CIFAR-10 dataset with batch size 4. First row: ground truth. Second row: 90\% gradient pruning. Third row: 80\% optimal pruning. 80\% optimal pruning has better training utility and better protection against reconstruction.}
    \label{fig:visualizeprunecifar}
\end{figure}

As shown in Figure \ref{fig:cifarprune}, our optimal pruning achieved higher reconstruction error than gradient pruning for the same level of training utility, with a pruning ratio of 70\% outperforming 90\% pruning in gradient pruning. Visual comparison (Figure \ref{fig:visualizeprunecifar}) also indicates better protection using our method.

\begin{figure}[h]
    \centering
    \includegraphics[width=\linewidth]{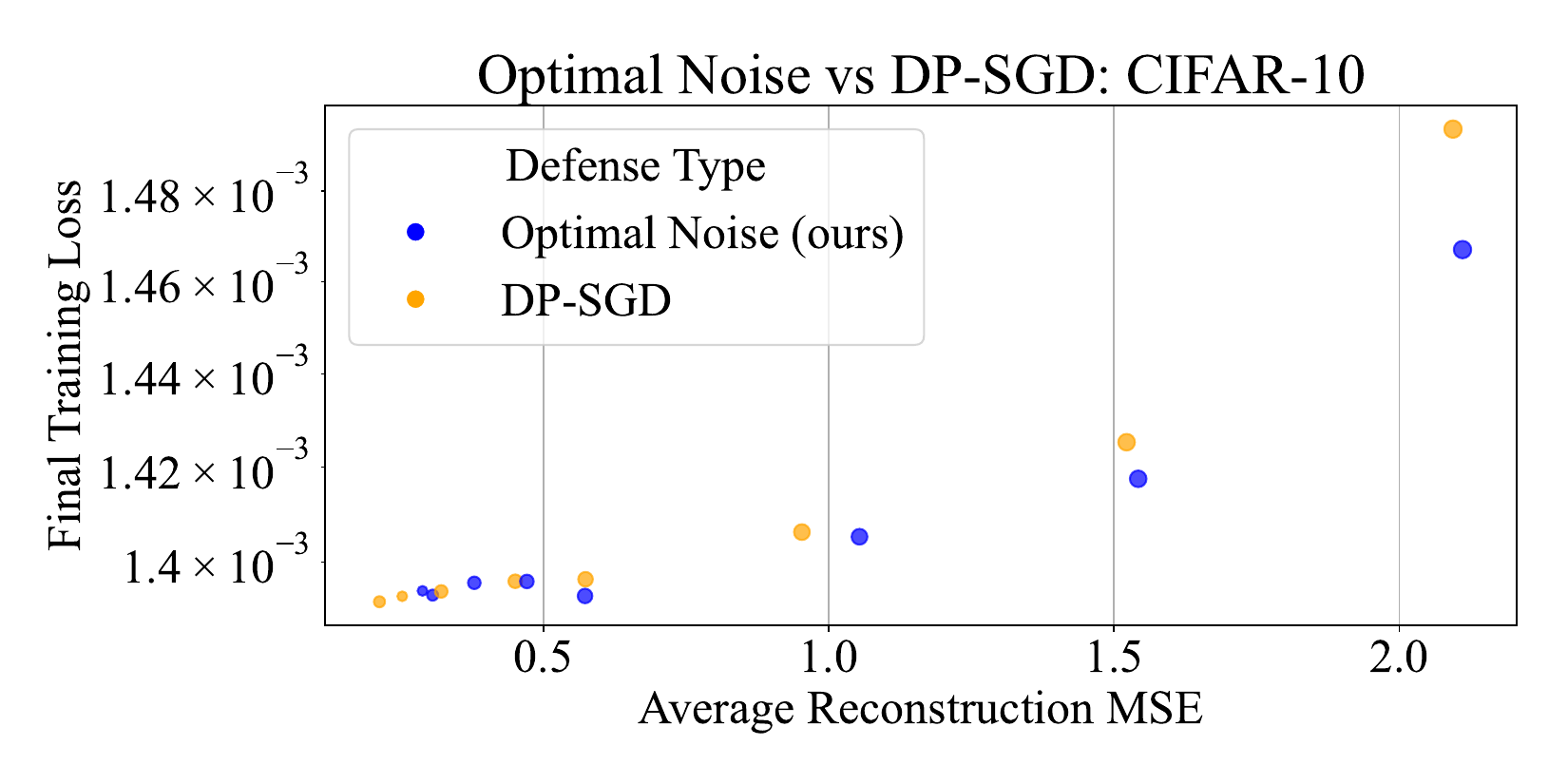}
    \caption{Comparison of optimal noise and DP-SGD on CIFAR-10. X-axis: average MSE. Y axis: Training loss on 8 samples.}
    \label{fig:cifarprunedpsgd}
\end{figure}

\subsubsection{Optimal Noise}
Reconstructing CIFAR-10 images is more challenging, with failures occurring at noise scales larger than $10^{-3}$. However, our noise method still offers a better privacy-utility trade-off, as shown in Figure \ref{fig:cifarprunedpsgd}. The advantage in training utility is more significant when the noise scale is larger. (Details in Appendix \ref{sec:additionalexperiments}.)

\clearpage
\newpage

%% file: tex/appendix.tex
\section{MISSING PROOFS}
\subsection{Proof of Theorem 1}
\begin{proof}
Recall the definition
\[
B_{\dD,S} = \min_{R:\mathbb{R}^d \to \mathbb{R}^m} \mathbb{E}_{\vx \sim \dD} \mathbb{E}_{\vy \sim S(g(\vx))} \left\| R(\vy) - \vx \right\|^2.
\]
By the Bayesian Cramér-Rao lower bound, for any reconstruction algorithm $R$ we have that:
\[
\mathbb{E}_{\vx \sim \dD} \mathbb{E}_{\vy \sim S(g(\vx))}  \left[ \left( R(\vy) - \vx \right) \left( R(\vy) - \vx \right)^\top \right] \succeq \mJ_B^{-1},
\]
where \( \mJ_B = \mJ_P + \mJ_D \) and \( \mJ_D \coloneq \mathbb{E}_{\vx\sim\dD} \left[ \mJ_F(\vx) \right] \).

Therefore:
\begin{align*}
\mathbb{E}_{\vx \sim \dD} \mathbb{E}_{\vy \sim S(g(\vx))} \left\| R(\vy) - \vx \right\|^2&=\tr\left(\mathbb{E}_{\vx \sim \dD} \mathbb{E}_{\vy \sim S(g(\vx))}  \left[ \left( R(\vy) - \vx \right) \left( R(\vy) - \vx \right)^\top \right]\right)\\
&\ge\tr\left(\mJ_B^{-1}\right).
\end{align*}

Since both \( \mJ_D \) and \( \mJ_P \) are Fisher information matrices and hence symmetric, we could apply Weyl's inequality to bound the eigenvalues of \( \mJ_B \). Let $\lambda_i$ denote sorted eigenvalues with $\lambda_1$ being the smallest and $\lambda_d$ the largest. For the eigenvalues \( \lambda_i \) of \( \mJ_B \), we have:
\[
\lambda_i(\mJ_B) \le \lambda_i(\mJ_D) + \lambda_1(\mJ_P).
\]
This implies that
\begin{align*}
\tr(\mJ_B^{-1}) &= \sum_{i=1}^{d} \frac{1}{\lambda_i(\mJ_B)} \\
&\ge \sum_{i=1}^{d} \frac{1}{\lambda_i(\mJ_D) + \lambda_1(\mJ_P)} \\
&\ge \frac{d^2}{\tr(\mJ_D) + d \cdot \lambda_1(\mJ_P)}.
\end{align*}
The last equation is from Cauchy's inequality since $\lambda_i(\mJ_D) + \lambda_1(\mJ_P)>0$.

Substituting \( \operatorname{tr}(\mJ_D) = \mathbb{E}_{x \sim \dD}\left[\operatorname{tr}(\mJ_F(\vx))\right] \), we obtain:
\[
\operatorname{tr}(\mJ_B^{-1}) \ge \frac{d^2}{\mathbb{E}_{x \sim \dD}\left[\operatorname{tr}(\mJ_F(\vx))\right] + d \cdot \lambda_1(\mJ_P)}.
\]

Thus, we have shown that
\[
B_{\dD,S} \ge \frac{d^2}{\mathbb{E}_{x \sim \dD}\left[\operatorname{tr}(\mJ_F(\vx))\right] + d \cdot \lambda_1(\mJ_P)}.
\]
\end{proof}

\subsection{Proof of Theorem 2}
To prove the theorem, we first need to calculate $\mJ_F(\vx)$:
\begin{lemma*}
\label{lem:fishernoise}
Let $\mJ_F(\vx)$ be the Fisher information matrix defined in Theorem 1. Let $\vy=S(\vx)$ be gradients defended with gradient noise using covariance matrix $\mSigma$. We have that:
\[
\mJ_F(\vx) = \nabla_{\vx} g(\vx) \mSigma^{-1} \nabla_{\vx} g(\vx)^\top.
\]
\end{lemma*}
\begin{proof}
$f:\RR^d\to\RR^n$ is the function from input data to model gradients.

Given \( \vy = g(\vx) + \epsilon \) with \(\epsilon \sim \mathcal{N}(0, \Sigma)\), we have the log-likelihood function:
\[
\log p(\vy | \vx) = -\frac{d}{2} \log(2\pi) - \frac{1}{2} \log |\Sigma| - \frac{1}{2} (\vy - g(\vx))^\top \Sigma^{-1} (\vy - g(\vx)).
\]

The gradient of the log-likelihood with respect to \(\vx\) is:
\[
\nabla_{\vx} \log p(\vy | \vx)^\top = (\vy - g(\vx))^\top \Sigma^{-1}  \nabla_{\vx} g(\vx)^\top.
\]

By definition:
\begin{align*}
    \mJ_F(\vx) &= \mathbb{E}_{\vy | \vx} \left[ \nabla_{\vx} \log p(\vy | \vx) \nabla_{\vx} \log p(\vy | \vx)^\top \right]\\
    &=\nabla_{\vx} g(\vx) \mathbb{E}_{\vy | \vx} \left[\Sigma^{-1} (\vy - g(\vx))   (\vy - g(\vx))^\top \Sigma^{-1}\right] \nabla_{\vx} g(\vx)^\top.
\end{align*}

Since \(\mathbb{E}_{\vy | \vx} \left[ (\vy - g(\vx)) (\vy - g(\vx))^\top \right] = \Sigma\),
\[
\mJ_F(\vx) = \nabla_{\vx} g(\vx) \Sigma^{-1} \nabla_{\vx} g(\vx)^\top.
\]

\end{proof}

Now we are ready to prove Theorem 2:

\begin{proof}
We want to minimize $\E_{\vx\sim\dD}\tr\left(\mJ_F(\vx)\right)$. By the lemma above, we have that
\[
\mJ_F(\vx) = \nabla_{\vx} g(\vx) \Sigma^{-1} \nabla_{\vx} g(\vx)^\top=
\sum_{i=1}^d  \frac{\norm{\nabla_\vx g_i(\vx)}^2}{\mSigma_{i,i}}
\]
when $\mSigma$ is diagonal.

The second-order utility for the defense equals:
\[
U_2(S,\Theta)=-\E_{\vx\sim\dD}\sum_{i=1}^d\left(\frac{\partial L(\Theta,\vx)}{\partial\evx_i}\right)^2 \emSigma_{i,i}=-\sum_{i=1}^d \E_{\vx\sim\dD}g_i(\vx)^2 \emSigma_{i,i},
\]
where the second equation is from the definition of $g_i(\vx)$. By Cauchy's inequality, we have that
\begin{align*}
\E_{\vx\sim\dD}\tr\left(\mJ_F(\vx)\right)&=\sum_{i=1}^d  \frac{\E_{\vx\sim\dD}\norm{\nabla_\vx g_i(\vx)}^2}{\mSigma_{i,i}}\\
&\ge\frac{\left(\sum_{i=1}^d \sqrt{\E_{\vx\sim\dD}\norm{\nabla_\vx g_i(\vx)}^2\cdot \E_{\vx\sim\dD}g_i(\vx)^2}\right)^2}{\sum_{i=1}^d \E_{\vx\sim\dD}g_i(\vx)^2 \emSigma_{i,i}}\\
&\ge \frac{1}{C}\left(\sum_{i=1}^d \sqrt{\E_{\vx\sim\dD}\norm{\nabla_\vx g_i(\vx)}^2\cdot \E_{\vx\sim\dD}g_i(\vx)^2}\right)^2.
\end{align*}
The first inequality holds when and only when
\[
\emSigma_{i,i}\propto \sqrt{\frac{\E_{\vx\sim\dD}\norm{\nabla_\vx g_i(\vx)}^2}{ \E_{\vx\sim\dD}g_i(\vx)^2}},
\]
and the second inequality holds after choosing the proper $\lambda(\vx)$ such that
\[
\emSigma_{i,i}=\lambda(\vx)\sqrt{\frac{\E_{\vx\sim\dD}\norm{\nabla_\vx g_i(\vx)}^2}{ \E_{\vx\sim\dD}g_i(\vx)^2}}
\]
satisfies $U_2(S,\Theta)=-C$.

Therefore $\tr\left(\mJ_F(\vx)\right)$ is minimized when any only when
\[
\emSigma_{i,i}=\lambda(\vx)\sqrt{\frac{\E_{\vx\sim\dD}\norm{\nabla_\vx g_i(\vx)}^2}{ \E_{\vx\sim\dD}g_i(\vx)^2}},
\]
for some $\lambda(\vx)$. The reconstruction error lower bound is maximized when the above applies.
\end{proof}

\subsection{Proof of Theorem 3}
We only need to replace $g(\vx)$ by $g_P(\vx)$ in Theorem 2. This is because the two defenses have the same utility measure when the noise covariance matrices are the same. Moreover, the lower bound only differs by changing $\nabla_\vx g(\vx)$ to $\nabla_\vx g_P(\vx)$.

Therefore Theorem 3 directly follows by replacing $g(\vx)$ with $g_P(\vx)$ in the optimal defense in Theorem 2.

\subsection{Proof of Theorem 4}
For generalized gradient pruning, we need to use mixed defense. Similar as Theorem 2, we first calculate $\tr(\mJ_{F,S_0(\cdot,r)}(x))$ in the mixed defense version of Theorem 1. For the noisy gradient pruning defense, the identifier $r$ is $\sA$, $S_0(\cdot,\sA)$ is noisy gradient pruning that prunes parameters in the set $\sA$. $\mathcal{R}$ is the distribution generating the pruning set. We find the optimal distribution $\mathcal{R}$.

\begin{lemma*}
\label{lem:fisherprune}
Let $\tr(\mJ_{F,S_{\prune,\mSigma,\sA}}(x))$ be the Fisher information matrix for mixed defense. Let $\vy=S(\vx)$ be the model gradients defended by noisy gradient pruning with covariance matrix $\mSigma=\epsilon\mI_d$ and pruning set $\sA\sim\mathcal{R}$. The Fisher information matrix $\tr(\mJ_{F,S_0(\cdot,r)}(x))$ defined in Theorem 1 equals:
\[
\tr(\mJ_{F,S_0(\cdot,r)}(x)) =\frac{1}{\epsilon} \sum_{i\notin\sA}\norm{\nabla_\vx g_i(\vx)}^2.
\]
\end{lemma*}
\begin{proof}
By the lemma used in the proof of Theorem 2, the left-hand side equals
\begin{align*}
\tr(\mJ_{F,S_0(\cdot,r)}(x)) =&\tr(\frac{1}{\epsilon}\nabla_{\vx_{\{1:d\}-\sA}}g(\vx)\nabla^\top_{\vx_{\{1:d\}-\sA}}g(\vx))\\
=&\frac{1}{\epsilon}\norm{\nabla_{\vx_{\{1:d\}-\sA}}g(\vx)}_F^2\\
=&\frac{1}{\epsilon} \sum_{i\notin\sA}\norm{\nabla_\vx g_i(\vx)}^2.
\end{align*}
The last equation is true since for each $i$, $\nabla_\vx g_i(\vx)$ corresponds to a column in $\nabla_{\vx_{\{1:d\}-\sA}}g(\vx)$.
\end{proof}

Now we can prove Theorem 4:
\begin{proof}
For the training utility, we have:
\begin{align*}
U_1(S,\Theta)=&\E_{\sA\sim\mathcal{R}}\EE_{x\sim\dD}\sum_{i\notin\sA}\left(\frac{\partial L(\Theta,\vx)}{\partial\evx_i}\right)^2\\
=&\E_{\sA\sim\mathcal{R}}\sum_{i\notin\sA}\EE_{x\sim\dD}\left(g_i(\vx)\right)^2\\
=&\sum_{i=1}^d P_{\sA\sim\mathcal{R}}(i\notin\sA)\E_{x\sim\dD}\left(g_i(\vx)\right)^2,
\end{align*}
where $P_{\sA\sim\mathcal{R}}(i\notin\sA)$ is the probability of $i\notin\sA$ when $\sA$ is sampled from $\mathcal{R}$.
With the given utility constraint $U_1(S,\Theta)\ge C$ we want to minimize $\E_{\sA\sim\mathcal{R}}\sum_{i\notin\sA}\EE_{x\sim\dD} \norm{\nabla_\vx g_i(\vx)}^2.$
By the previous lemma, we have:
\[\E_{\sA\sim\mathcal{R},\vx\sim\dD}\tr(\mJ_{F,\sA}(\vx))=\E_{\sA\sim\mathcal{R}}\sum_{i\notin\sA}\EE_{x\sim\dD} \norm{\nabla_\vx g_i(\vx)}^2=\sum_{i=1}^d P_{\sA\sim\mathcal{R}}(i\notin\sA)\EE_{x\sim\dD} \norm{\nabla_\vx g_i(\vx)}^2.\]
For all $i$, we have that $0\le P_{\sA\sim\mathcal{R}}(i\notin\sA)\le 1$, $\EE_{x\sim\dD} \norm{\nabla_\vx g_i(\vx)}^2>0$, and $\E_{x\sim\dD}\left(g_i(\vx)\right)^2>0$. Therefore, for the optimal defense minimizing $\E_{\sA\sim\mathcal{R},\vx\sim\dD}\tr(\mJ_{F,\sA}(\vx))$, the two restrictions apply:
\begin{itemize}
    \item $U_1(S,\Theta)=C.$
    \item If $P_{\sA\sim\mathcal{R}}(i\notin\sA)>0$, then for any $j$ such that
    \[\frac{\EE_{x\sim\dD} \norm{\nabla_\vx g_i(\vx)}^2}{\E_{x\sim\dD}\left(g_i(\vx)\right)^2}>\frac{\EE_{x\sim\dD} \norm{\nabla_\vx g_j(\vx)}^2}{\E_{x\sim\dD}\left(g_j(\vx)\right)^2},\]
    we must have $P_{\sA\sim\mathcal{R}}(j\notin\sA)=1$.
\end{itemize}
If any of the above does not apply, we could trivially modify $\mathcal{R}$ to improve the lower bound while staying within the utility budget. If $U_1(S,\Theta)>C$ and a parameter is pruned with probability smaller than 1, we could slightly increase the probability of pruning that parameter. This slightly decreases utility but yields a better reconstruction lower bound. If the second restriction does not apply, we could increase the probability of pruning $j$ and decrease the probability of pruning $i$ to have the same training utility and obtain a higher reconstruction error lower bound. When the restrictions apply, the resulting defense follows our theorem.
\end{proof}

Furthermore, in locally optimal gradient pruning, adding our optimal noise instead of standard Gaussian noise after the gradient pruning step yields the same optimal defense (optimal pruning set). 
To show the claim, notice that in this case, the Fisher information matrix in the lower bound is:
$$\tr(\mJ_{F,S_{\prune,\mSigma,\sA}}(x)) \approx\frac{1}{\epsilon} \sum_{i\notin\sA}\norm{\nabla_\vx g_i(\vx)}|g_i(\vx)|.$$
Since the utility function remains the same, the index for optimal pruning is now
\[k_i=\frac{\norm{\nabla_\vx g_i(\vx)}|g_i(\vx)|}{|g_i(\vx)|^2}=\frac{\norm{\nabla_\vx g_i(\vx)}}{|g_i(\vx)|},\]
which is equivalent to the locally optimal pruning with standard noise.

\subsection{Proof of Lemma 1}
\begin{proof}
Notice that for any given $\vy\in\R^d$,
\begin{equation*}
\frac{\partial g(\vx+\alpha\vy)}{\partial\alpha}=\nabla_\vx g(\vx+\alpha\vy)\odot\vy,
\end{equation*}
where $\odot$ represents element-wise multiplication.
Therefore
\[
\norm{\frac{\partial g(\vx+\alpha\vy)}{\partial\alpha}}_{\alpha=0}=\norm{\nabla_\vx g(\vx)\odot\vy}.
\]
When $\vy\sim\mathcal{N}(0,\mI_d)$, $\norm{\nabla_\vx g(\vx)\odot\vy}$ follows a normal distribution with mean 0 and variance $\norm{\nabla_\vx g(\vx)}^2$.

Therefore we have that 
$$\frac{\norm{\frac{\partial g(\vx+\alpha\vy)}{\partial\alpha}}_{\alpha=0}}{\norm{\nabla_\vx g(\vx)}}\sim\mathcal{N}(0,1),$$
furthermore,
$$A:=\frac{\sum_{i=1}^k\norm{\frac{\partial g(\vx+\alpha\vx_i)}{\partial\alpha}}_{\alpha=0}}{\norm{\nabla_\vx g(\vx)}}\sim\chi^2(k).$$
Since $\E(A)=k$ and $\text{Var}(A)=2k$, by Markov's inequality we have that
\begin{align*}
    P(|A-k|>k\epsilon)\le \frac{2k}{k^2\epsilon^2}=\frac{2}{k\epsilon^2}.
\end{align*}
Since 
\[\left|\norm{\nabla_\vx f(\vx)}^2-\frac{1}{k} \sum_{j=1}^k \norm{\frac{\partial f_i(\vx + \alpha \vx_j)}{\partial\alpha}}^2_{\alpha=0}\right|\le \epsilon \norm{\nabla_\vx f(\vx)}^2\]
is equivalent to $|A-k|>k\epsilon$, we finished the proof.
\end{proof}

\section{OPTIMAL DEFENSE WITH ReLU ACTIVATION}
\label{sec:thmrelu}
In this section, we briefly discuss how we modify our optimal defenses when $\E_{\vx\sim\dD}g_i(\vx)^2=0$. The defenses for the entries where the gradients are not 0 are the same, but we deal with entries with gradient 0 separately.

For Gradient Pruning, pruning gradients that are 0 do not affect the defended gradients. Therefore, we focus on analyzing how we apply Optimal Gradient Noise.  By the definition of second-order utility, the second-order utility is unaffected by the noise scale added on parameters with gradient 0. If we follow the proof of Theorem 2, the reconstruction lower bound scales negatively with 
\[\sum_{i=1}^d\frac{\E_{\vx\sim\dD}\norm{\nabla_\vx g_i(\vx)}^2}{\emSigma_{i,i}}.\]

Since $\E_{\vx\sim\dD}g_i(\vx)^2=0$, we have that $g_i(\vx)$ is constant on the support of $\dD$ so $\E_{\vx\sim\dD}\norm{\nabla_\vx g_i(\vx)}^2=0.$ Therefore, the noise scale does not affect the reconstruction lower bound and any noise scale for these parameters are optimal. This typically happens when the model has activation functions that are constant on an interval (e.g. ReLU). Nonetheless, since a larger noise scale typically decreases training utility, the optimal method would be not to add noise to the gradients.

However, the above case does not completely cover problems in the locally optimal defenses. In the locally optimal versions, we used the values at $\vx$ as approximations of expectations. Therefore we might have $\norm{\nabla_\vx g_i(\vx)}^2>0$. In this case, theoretically we should set the noise to be as large as possible, which deviates from reality. To resolve this problem, a good solution would be to set an upper limit to the noise scale and clip the noise scales added to the gradients.

To summarize, for optimal gradient pruning, pruning gradients with a scale of 0 is trivial. For optimal gradient noise (and its application in DP-SGD), we set an upper limit to the noise scale.

\section{EXPERIMENT DETAILS}
For all datasets and algorithms, we used the implementation in Algorithm 1 with $k=10$ as the number of samples and $c=10^{-6}$ as the small constant. When comparing our optimal noise with DP-SGD, we applied clipping threshold 1 for DP-SGD and included the same clipping step before adding our optimal noise.

\subsection{Experiments with MNIST}
For experiments on the MNIST dataset, we used a Convolutional Neural Network with 120k parameters. To avoid the special case where the model utilizes the ReLU activation function, we use the LeakyReLU activation function instead. We trained the model on a subset of size 4096 from the whole dataset and used SGD algorithms with batch size 64. We simulated 4 clients each possessing one-fourth of the dataset (1024 samples). When training, each mini-batch contains data from all 4 clients, with each client providing 16 samples to form a 64-sample mini-batch. The gradients are computed and defended separately and then averaged to be provided to the central server. For the training process with DP-SGD (or our optimal noise) applied, we used the Adam optimizer with learning rate $10^{-3}$. For the training process with gradient pruning (or our optimal pruning) applied, we used Adam with learning rate $5\times 10^{-4}$. For the reconstruction process, we used the Inverting Gradients algorithm with a budget of 2000 updates. The experiments are conducted on a Nvidia RTX 2050. 

For the scatter plot, we trained the model on 1 batch of 64 samples for 5 gradient descent updates. Each gradient used for update is the average of 4 gradients calculated and defensed separately from 4 clients. We used the Adam optimizer with learning rate $10^{-3}$.

\subsection{Experiments with CIFAR-10}
For experiments on the CIFAR-10 dataset, we used a 2.9M parameter ConvNet with the LeakyReLU activation. For the scatter plot, we measured the training utility by training on a batch of 8 samples. For comparing DP-SGD with our optimal noise, we used the Adam optimizer for 5 steps with the learning rate being $10^{-4}$. For comparing gradient pruning with our optimal pruning, we used Adam with the learning rate being $5\times10^{-5}$ (for slower convergence). Every update we simulated 4 clients each calculating and defending gradients separately like in MNIST. For reconstruction, we reconstructed the images 2-at-a-time from the shared model gradients using the Inverting Gradients algorithm with a budget of 1000 updates. The experiments are conducted on a Nvidia RTX 4090D.

\section{ADDITIONAL EXPERIMENTS}
\label{sec:additionalexperiments}
\subsection{Effect of Noise Scale on DP-SGD in CIFAR-10}
\begin{figure}
    \centering
    \includegraphics[width=0.7\linewidth]{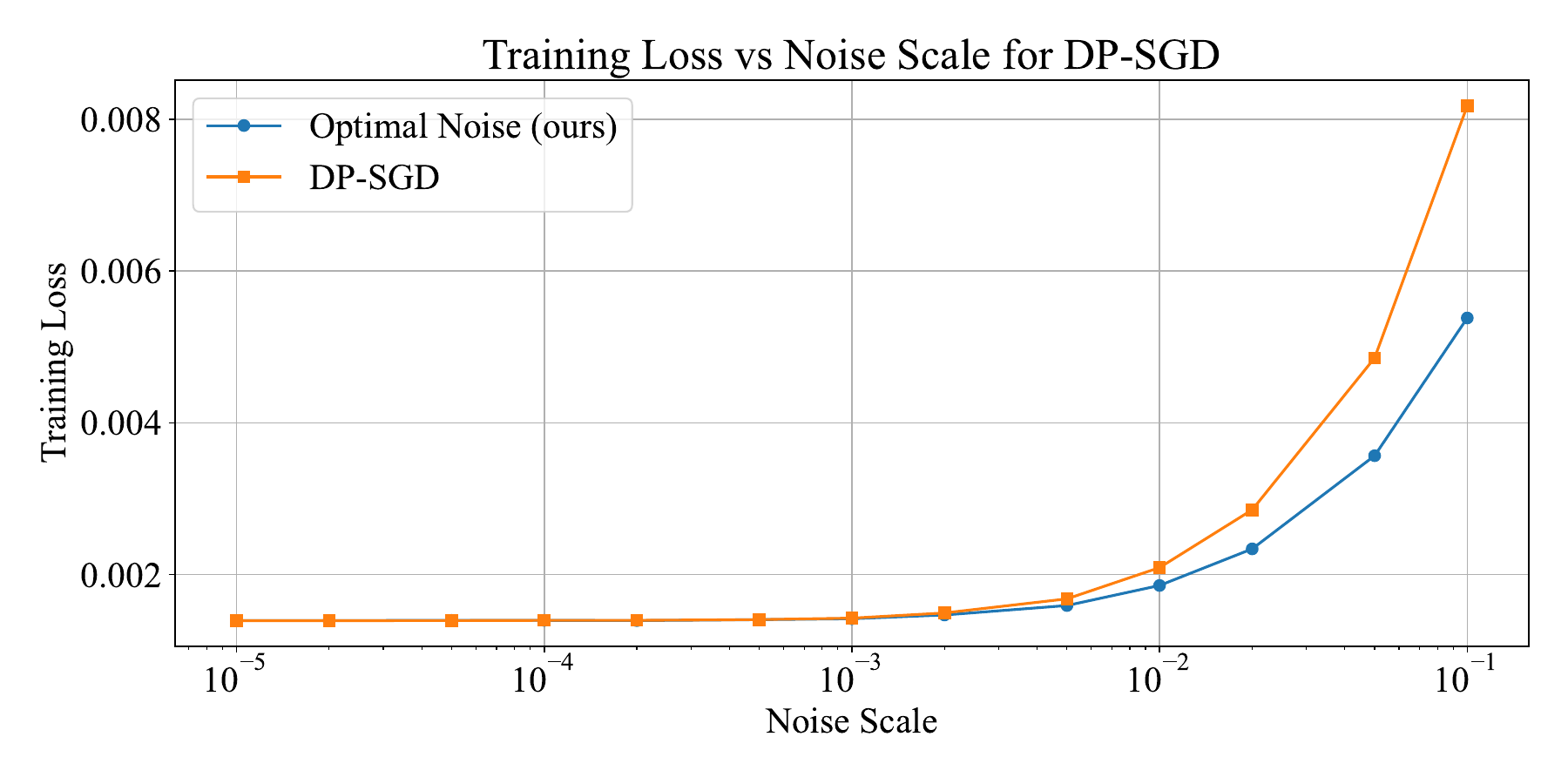}
    \caption{Effect of Noise Scale on Training Utility.}
    \label{fig:cifarnoisecurve}
\end{figure}

\begin{figure}
    \centering
    \includegraphics[width=0.7\linewidth]{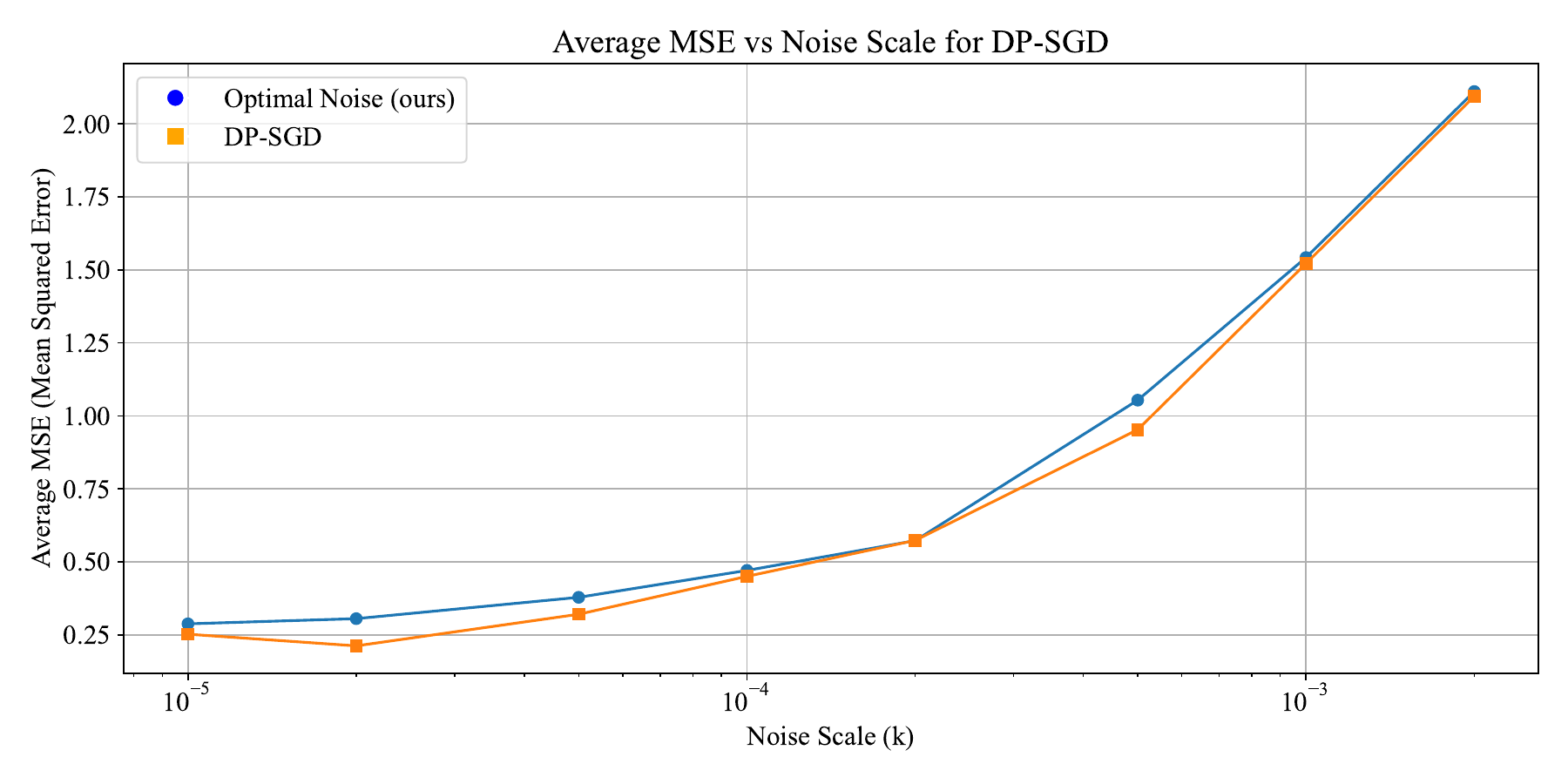}
    \caption{Effect of Noise Scale on Average Reconstruction MSE.}
    \label{fig:cifarnoisemsecurve}
\end{figure}

Since our experiments in the previous sections only cover small noise scales, we visualize the effects of the larger noise scales on CIFAR-10 in Figure \ref{fig:cifarnoisecurve} and \ref{fig:cifarnoisemsecurve}. With higher noise scales, it is clearer that our optimal noise has higher training utility and (slightly) higher reconstruction error than DP-SGD.

\begin{table}[h!]
\centering
\begin{tabular}{|c|c|}
\hline
\textbf{Layer}                & \textbf{Parameters} \\ \hline
conv\_layers.0.weight         & 288                \\ \hline
conv\_layers.0.bias           & 32                 \\ \hline
conv\_layers.3.weight         & 18,432             \\ \hline
conv\_layers.3.bias           & 64                 \\ \hline
fc\_layers.1.weight           & 100,352            \\ \hline
fc\_layers.1.bias             & 32                 \\ \hline
fc\_layers.3.weight           & 320                \\ \hline
fc\_layers.3.bias             & 10                 \\ \hline
\end{tabular}
\caption{Number of parameters in the neural network.}
\label{tab:layer_params}
\end{table}

\begin{figure}
    \centering
    \includegraphics[width=0.45\linewidth]{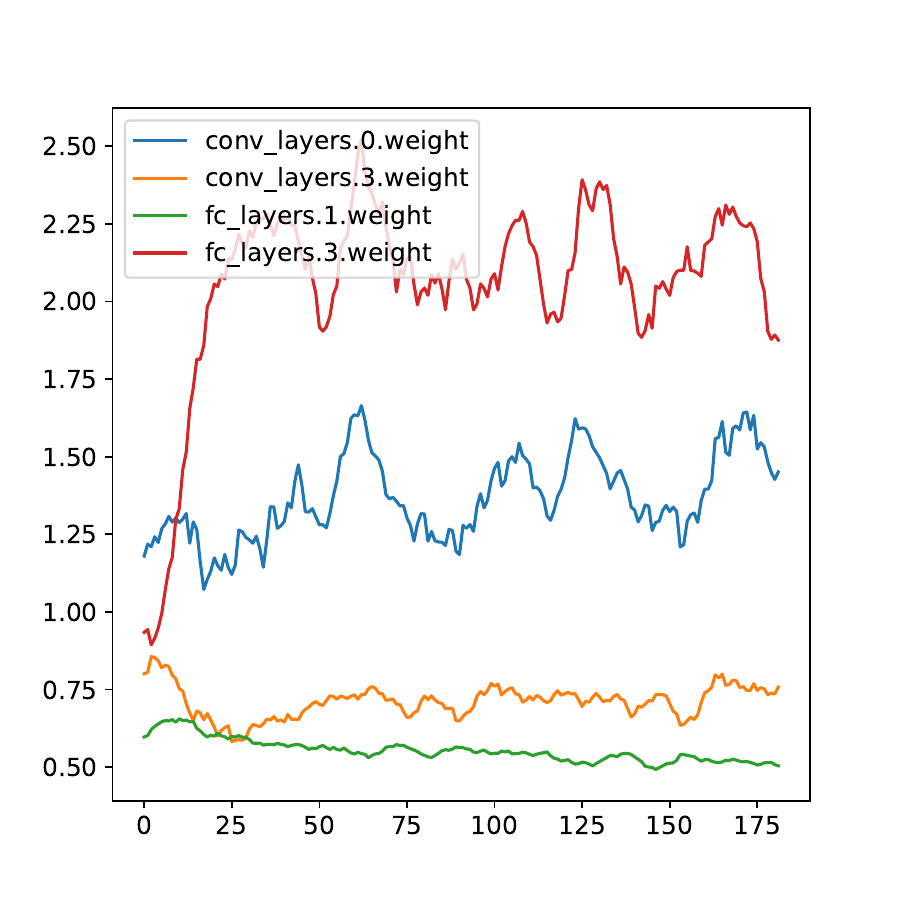}
    \includegraphics[width=0.45\linewidth]{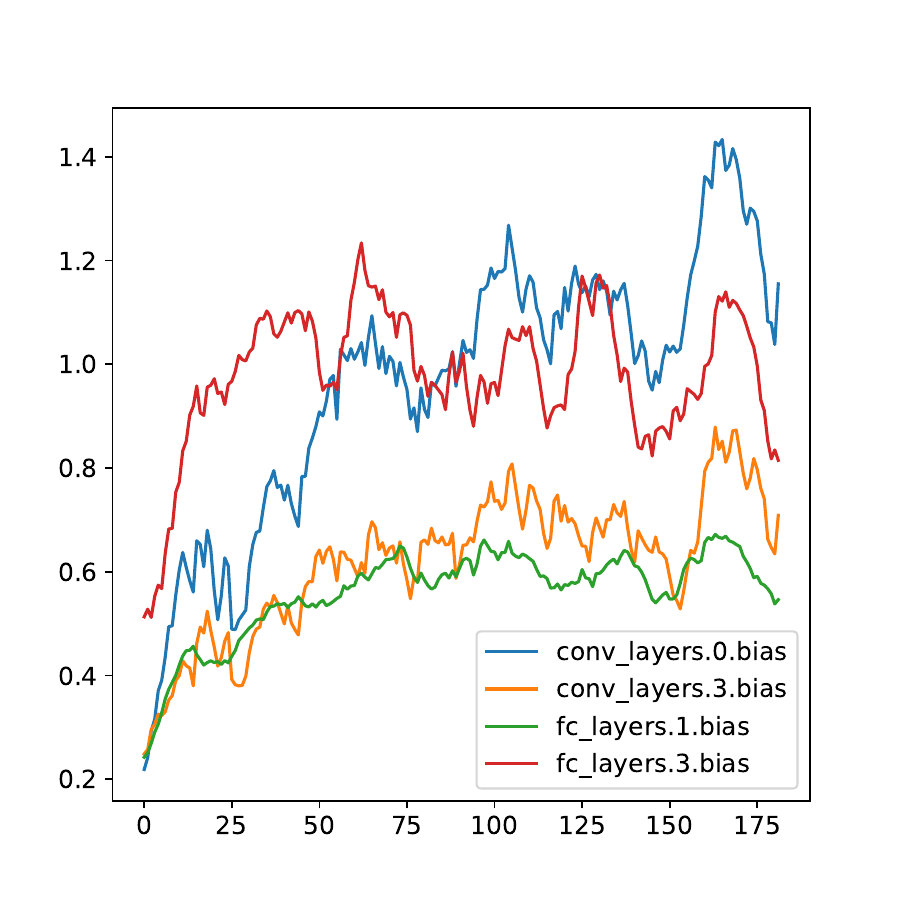}
    \caption{Average noise each layer smoothed by window size 10. Left: weight matrices. Right: bias matrices.}
    \label{fig:MNISTnorm}
\end{figure}

\subsection{Average Optimal Noise Through the Training Process}
We additionally show how our method differs from DP-SGD by visualizing average noise added to each layer during the training process. The experiment utilized the MNIST dataset. As in Figure \ref{fig:MNISTnorm}, some layers are given significantly higher noise than other layers. We additionally include how many parameters each layer contains in Table \ref{tab:layer_params} and a code snippet of the model in Pytorch for reference.

\begin{lstlisting}[language=python]
class SimpleConvNet(nn.Module):
    def __init__(self, num_classes=10, num_channels=1):
        super(SimpleConvNet, self).__init__()

        # Convolutional layers
        self.conv_layers = nn.Sequential(
            # conv_layers.0 (Conv2d)
            nn.Conv2d(num_channels, 32, kernel_size=3, padding=1),  
            nn.LeakyReLU(), 
            nn.MaxPool2d(2),

            # conv_layers.3 (Conv2d)
            nn.Conv2d(32, 64, kernel_size=3, padding=1),            
            nn.LeakyReLU(),  
            nn.MaxPool2d(2) 
        )
        
        # Fully connected layers
        self.fc_layers = nn.Sequential(
            nn.Flatten(),    
            # fc_layers.1 (Linear)
            nn.Linear(64 * 7 * 7, 32), 
            nn.LeakyReLU(),   
            # fc_layers.3 (Linear)
            nn.Linear(32, num_classes)                               
        )
    
    def forward(self, x):
        x = self.conv_layers(x)
        x = self.fc_layers(x)
        return x
\end{lstlisting}